\newcommand {\rf} {\mathit{rank}}

\newcommand {\lingconc} {\mathcal{S}}
\newcommand {\Pe} {\textsf{P}}
\newcommand {\Ra} {\textsf{R}}

\newcommand {\ent} {\mathrel{{\scriptstyle\mid\!\sim}}}

\newcommand {\nott} {\lnot}

\newcommand {\esiste} {\exists}
\newcommand {\la} {\langle}
\newcommand {\ra} {\rangle}
\newcommand {\sx} {\langle}
\newcommand {\dx} {\rangle}
\newcommand {\incluso} {\subseteq}
\newcommand {\appartiene} {\in}
\newcommand {\emme} {\mathcal{M}}

\newcommand {\unione} {\cup}

\newcommand {\tc} {\mid}
\newcommand {\vuoto} {\emptyset}
\newcommand {\WW} {\mathcal{W}}

\newcommand {\diverso} {\neq}

\newcommand{\tip}{{\bf T}}

\newcommand{\alc}{\mathcal{ALC}}
\newcommand{\alct}{\mathcal{ALC}+\tip}
\newcommand{\alctmin}{\mathcal{ALC}+\tip_{min}}

\newcommand{\alctr}{\mathcal{ALC}^{\Ra}\tip}

\newcommand{\modelsshiqrt}{\models_{\scriptsize\mathcal{SHIQ}^{\scriptsize{\Ra}}\tip}}

\newcommand {\bbox}{\square}

\newcommand{\be}{\begin{enumerate}}
\newcommand{\ee}{\end{enumerate}}

\newcommand{\hide}[1]{}

\newcommand{\subsud}{\utilde{\sqsubset}}

\newcommand{\sqset}{\sqsubseteq}

\newcommand{\shiqrt}{\mathcal{SHIQ}^{\Ra}\tip}
\newcommand{\shiqpt}{\mathcal{SHIQ}^{P}\tip}
\newcommand{\shiq}{\mathcal{SHIQ}}

\documentclass[runningheads]{llncs}
\usepackage{times}
\usepackage{graphicx}
\usepackage{latexsym}

\usepackage{multicol,lipsum}
\usepackage{undertilde}
\usepackage{amsmath}
\usepackage{microtype}
\usepackage{amssymb}
\usepackage{color}

\begin{document}

\title{Rational closure in $\shiq$}

\author{\vspace{-0.15cm}
Laura Giordano \inst{1} \and Valentina Gliozzi \inst{2} \and Nicola Olivetti \inst{3} \and Gian Luca Pozzato \inst{2}}


\authorrunning{L. Giordano, V. Gliozzi, N. Olivetti, G.L. Pozzato}

\institute{DISIT - U. Piemonte
Orientale, Alessandria, Italy - \email{laura@mfn.unipmn.it}
\and Dip. di Informatica - Univ. di Torino, Italy - \email{\{gliozzi,pozzato\}@di.unito.it}
\and Aix-Marseille Universit\'e, CNRS, France \email{nicola.olivetti@univ-amu.fr }
\vspace{-0.6cm}
}

\maketitle

\begin{abstract}\vspace{-0.1cm}
We define a notion of rational closure for the logic $\shiq$, which does not enjoys the finite model property, building on the notion of rational closure introduced by Lehmann and Magidor in \cite{whatdoes}.
We provide a semantic characterization of rational closure in $\shiq$ in terms of a preferential semantics, based on  a finite rank characterization of minimal models. We show that the rational closure of a TBox can be computed in  \textsc{ExpTime} using entailment in $\shiq$. 
\vspace{-0.2cm}
\end{abstract}

\section{Introduction}
\vspace{-0.25cm}

Recently, a large amount of work  has  been done in order to extend the basic formalism of Description Logics (for short, DLs)  with nonmonotonic reasoning features \cite{Straccia93,baader95b,donini2002,eiter2004,lpar2007,AIJ,kesattler,sudafricaniKR,bonattilutz,casinistraccia2010,rosatiacm,hitzlerdl}; the purpose of these extensions is that of allowing reasoning about \emph{prototypical properties} of individuals or classes of individuals. In these extensions one can represent, for instance, knowledge expressing the fact that the hematocrit level is \emph{usually} under 50\%, with the exceptions of newborns and of males residing at high altitudes, that have usually much  higher levels (even over 65\%). Furthermore, one can infer that an individual enjoys all the \emph{typical} properties of the classes it belongs to. As an example, in the absence of information that Carlos and the son of Fernando are either newborns or adult males living at a high altitude, one would assume that the hematocrit levels of Carlos and Fernando's son are under 50\%. 
This kind of inferences apply to  individual explicitly named in the knowledge base as well as to individuals implicitly introduced by relations among individuals (the son of Fernando).

In spite of the number of works in this direction, finding a solution to the problem of extending DLs for reasoning about prototypical properties seems far from being solved. 
The most well known semantics for nonmonotonic reasoning have been used to the purpose, 
from default logic \cite{baader95b}, to circumscription  \cite{bonattilutz}, to Lifschitz's  nonmonotonic logic MKNF
\cite{donini2002,rosatiacm}, to preferential reasoning \cite{lpar2007,sudafricaniKR,AIJ}, to rational closure \cite{casinistraccia2010,CasiniJAIR2013}.

In this work, we focus on rational closure and, specifically, on the rational closure for $\shiq$.
The interest of rational closure in DLs is that it provides a significant and reasonable nonmonotonic inference mechanism, still remaining computationally inexpensive.
As shown for $\alc$ in  \cite{casinistraccia2010}, 
its complexity can be expected not to exceed the one of the underlying monotonic DL. 
This is a striking difference with most of the other approaches to nonmonotonic reasoning in DLs mentioned above, with some exception such as \cite{rosatiacm,hitzlerdl}.
More specifically, we define a rational closure for the logic $\shiq$, building on the notion of rational closure in \cite{whatdoes} for propositional logic. This is a difference with respect to the rational closure construction introduced in \cite{casinistraccia2010} for $\alc$, which  is more similar to the one by  Freund \cite{freund98} for propositional logic
(for propositional logic, the two definitions of rational closure are shown to be equivalent \cite{freund98}).
We provide a semantic characterization of rational closure in $\shiq$ in terms of a preferential semantics,  
by generalizing to $\shiq$ the results for rational closure in $\alc$ presented in \cite{dl2013}. This generalization is not trivial, since $\shiq$ lacks a crucial property of $\alc$, the finite model property  \cite{HorrocksIGPL00}.
Our  construction exploits an extension of $\shiq$ with a typicality operator $\tip$, that selects the most typical instances of a concept $C$, $\tip(C)$. We define a {\em minimal model semantics}
and a notion of minimal entailment  for the resulting logic, $\shiqrt$, 
and we show that the inclusions belonging to the rational closure of a TBox are those minimally entailed by the TBox, when restricting to {\em canonical} models. This result exploits a characterization of minimal models, showing that we can restrict to models with finite ranks.
We also show that the rational closure construction of a TBox can be done exploiting entailment in $\shiq$,
without requiring to reason in $\shiqrt$, and that the problem of deciding whether an inclusion belongs to the rational closure of a TBox  is in \textsc{ExpTime}.

Concerning ABox reasoning, because of the interaction between individuals (due to roles) it is not possible to separately assign a unique minimal rank to each individual and  alternative minimal ranks must be considered. 
We end up with a kind of \emph{skeptical} inference with respect to the ABox,
whose complexity in \textsc{ExpTime} as well.
\normalcolor




\normalcolor

\vspace{-0.15cm}

\section{A nonmonotonic extension of $\shiq$}\label{sezione shiqrt}
\vspace{-0.25cm}
Following the approach
in \cite{FI09,AIJ}, we  introduce an extension of  $\shiq$ \cite{HorrocksIGPL00} with a typicality operator $\tip$ in order to express typical inclusions, obtaining the logic $\shiqrt$. The intuitive idea is to allow concepts of the form $\tip(C)$,  whose intuitive meaning is that
$\tip(C)$ selects the {\em typical} instances of a concept $C$. We can therefore distinguish between the properties that
hold for all instances of $C$ ($C \sqsubseteq D$), and those that only hold for the typical
such instances ($\tip(C) \sqsubseteq D$).
Since we are dealing here with rational closure, we attribute to $\tip$ properties of rational consequence relation \cite{whatdoes}. 
We consider an alphabet of concept names $\mathcal{C}$, role names
$\mathcal{R}$, transitive roles $\mathcal{R}^+ \subseteq \mathcal{R}$,  and individual constants $\mathcal{O}$.
Given $A \in \mathcal{C}$, $S \in \mathcal{R}$, and $n \in \mathbb{N}$  we define:

\vspace{0.1cm}
\begin{footnotesize}
\noindent $C_R:= A \tc \top \tc \bot \tc  \nott C_R \tc C_R \sqcap C_R \tc C_R \sqcup C_R \tc \forall S.C_R \tc \exists S.C_R \tc (\geq n S.C_R) \tc (\leq n S.C_R)$  \ \ \ \ \ \ 
   $C_L:= C_R \tc  \tip(C_R)$
  \hfill $S:=R \tc R^-$
\end{footnotesize}
\vspace{0.1cm}

\noindent  As usual, we assume that transitive roles cannot be used in number restrictions \cite{HorrocksIGPL00}.
    A KB is a pair (TBox, ABox). TBox contains a finite set
of  concept inclusions  $C_L \sqsubseteq C_R$ and role inclusions $R \sqsubseteq S$. ABox
contains assertions of the form $C_L(a)$ and $S(a,b)$, where $a, b \in
\mathcal{O}$. 

\hide{
\noindent A first semantic characterization of $\tip$ can be given by means of a set of postulates that are essentially a reformulation of axioms and rules of nonmonotonic entailment in rational logic \Ra: in this respect, the \tip-assertion $\tip(C) \sqsubseteq D$ is equivalent to the conditional assertion $C \ent D$ in $\Ra$\footnote{This can be easily proven given  Proposition $5.1$ of \cite{FI09}  that proves the equivalence between the weaker logic $\alct$ in which $f_\tip$ satisfies $(f_\tip-1)-(f_\tip-5)$ above but does not satisfy $(f_\tip-\Ra)$ and the KLM logic $\Pe$ which is weaker than \Ra.}. Given a domain $\Delta$  and a valuation function $I$ one can define the function
$f_\tip(S)$ that selects  the {\em typical} instances of $S$, and in
case $S = C^I$ for a concept $C$, it selects the typical instances
of $C$. In this semantics, $(\tip(C))^I = f_\tip(C^I)$, and
$f_\tip$ has the following intuitive properties for all subsets $S$ of
$\Delta$:

\begin{definition}[Semantics of $\tip$ with selection function]\label{Semantics
with f_tip} A model is any structure $$\langle \Delta, f_\tip, I\rangle$$
where: 
\begin{itemize}
\item $\Delta$ is the domain; 
\item $f_\tip : Pow(\Delta) \longmapsto Pow(\Delta)$ is a function satisfying the
following properties (given $S \subseteq \Delta$):
\end{itemize}
\begin{quote}
 \begin{footnotesize}
$ (f_\tip-1) \ f_\tip(S) \subseteq S $ \\  
$ (f_\tip-2) \ \mbox{if} \ S \not=\emptyset \mbox{, then also} \ f_\tip(S) \not=\emptyset $ \\
$ (f_\tip-3) \ \mbox{if} \ f_\tip(S) \subseteq R, \mbox{then} \ f_\tip(S)=f_\tip(S \cap R) $ \\ 
$ (f_\tip-4) \ f_\tip(\bigcup S_i) \subseteq \bigcup f_\tip (S_i) $ \\
$ (f_\tip-5) \ \bigcap f_\tip(S_i) \subseteq f_\tip(\bigcup S_i) $ \\
$  (f_\tip-\Ra) \ \mbox{if} \ f_\tip(S) \cap R \not=\emptyset, \mbox{then} \ f_\tip(S \cap R) \subseteq f_\tip(S) $ \\
 \end{footnotesize}
 \end{quote}
 \begin{itemize}
\item $I$ is the extension function that
maps each extended concept $C$ to $C^I \subseteq \Delta$, and each role $R$
to a $R^I \subseteq \Delta \times \Delta$ as follows:
\begin{itemize}
  \item $I$ maps   each role $R \in \mathcal{R}$ to its extension $R^I$;
   \item $I$ maps  each atomic concept $A \in \mathcal{C}$ to its extension $A^I$;
  \item $I$ is extended to complex concepts as follows:
\begin{itemize}
   \item $\top^I=\Delta$
   \item $\bot^I=\emptyset$
   \item $(C \sqcap D)^I$=$C^I \cap D^I$
   \item $(\esiste R.C)^I$=$\{x \in \Delta \tc \exists y \in C^I: (x,y) \in R^I \}$
   \item $(\geq n R.C)^I=\{x \in \Delta \ \mbox{s.t.} \ \tc \{y \in \Delta \ \mbox{s.t.} \ (x^I,y^I) \in R^I \ \mbox{and} \ y^I \in C^I\}\tc \geq n \}$
   \item $(\leq n R.C)^I=\{x \in \Delta \ \mbox{s.t.} \ \tc \{y \in \Delta \ \mbox{s.t.} \ (x^I,y^I) \in R^I \ \mbox{and} \ y^I \in C^I\}\tc \leq n \}$
   \item $(\tip(C))^I = f_\tip(C^I)$
\end{itemize}
  \item $I$ is extended to inverse and transitive roles as follows:
  \begin{itemize}
   \item $(R^-)^I=\{(x^I,y^I) \tc (y^I,x^I) \in R^I\}$
   \item for $R \in \mathcal{R}^+$, if $(x^I,y^I) \in R^I$ and $(y^I,z^I) \in R^I$, then $(x^I,z^I) \in R^I$.
  \end{itemize}
\end{itemize}
\end{itemize}
\end{definition}

\noindent ($f_\tip-1$) enforces that typical elements of $S$
belong to $S$. ($f_\tip-2$) enforces that if there are elements in
$S$, then there are also {\em typical} such elements. ($f_\tip-3$) expresses a weak form of monotonicity, namely {\em
cautious monotonicity}. The next
properties constraint the behavior of $f_\tip$ with respect to $\cap$ and
$\cup$ in such a way that they do not entail monotonicity.
 Last,   $(f_\tip-\Ra)$ corresponds to rational monotonicity, and forces again a form of
monotonicity: if there is a typical $S$ having the property $R$, then
all typical $S$ and $R$s inherit the properties of typical $S$s. 
}

\noindent The semantics of $\shiqrt$ is formulated in terms of
rational models:
 ordinary models of $\shiq$ are equipped with a \emph{preference relation} $<$ on
the domain, whose intuitive meaning is to compare the ``typicality''
of domain elements, that is to say, $x < y$ means that $x$ is more typical than
$y$. Typical instances of a concept $C$ (the instances of
$\tip(C)$) are the instances $x$ of $C$ that are minimal with respect
to the preference relation $<$ (so that there is no other instance of $C$
preferred to $x$)\footnote{As for the logic $\alctr$ in \cite{ecai2010DLs}, an alternative semantic characterization of $\tip$ can be given by means of a set of postulates that are essentially a reformulation of the properties of rational consequence relation \cite{whatdoes}.}.


In the following definition we introduce the notion of 

\begin{definition}[Semantics of $\shiqrt$]\label{semalctr} A $\shiqrt$ model 
\footnote{In this paper, we follow the terminology in \cite{whatdoes}
for preferential and ranked models, and we use the term ``model" to denote an an interpretation.}
$\emme$ is any
structure $\langle \Delta, <,I \rangle$ where: 
\begin{itemize}
\item $\Delta$ is the domain;   
\item  $<$ is an irreflexive, transitive, well-founded, and modular   relation over
$\Delta$; 
 \item $I$ is the extension function that maps each
concept $C$ to $C^I \subseteq \Delta$, and each role $R$
to  $R^I \subseteq \Delta^I \times \Delta^I$. For concepts of
$\shiq$, $C^I$ is defined as usual. For the $\tip$ operator, we have
$(\tip(C))^I = Min_<(C^I)$,  where
$Min_<(S)= \{u: u \in S$ and $\nexists z \in S$ s.t. $z < u \}$.
\end{itemize}
\end{definition}

\noindent We say that an irreflexive and transitive relation $<$ is: 
\begin{itemize}
\item {\em modular} if, for all $x,y,z \in \Delta$, if $x < y$ then $x < z$ or $z < y$  \cite{whatdoes};
\item   {\em well-founded} if,
for all  $S \subseteq \Delta$, for all $x \in S$, either $x \in Min_<(S)$ or $\exists y \in  Min_<(S)$ such that $y < x$. 
\end{itemize}

\noindent It can be proved that an irreflexive and transitive relation $<$ on $\Delta$ is well-founded if and only if there 
are no infinite descending chains $\ldots x_{i+1} <^* x_i <^* \ldots <^* x_0$ of elements of $\Delta$
(see Appendix \ref{appendicenicola}).

In \cite{whatdoes} it is shown that,  for a strict partial order $<$ over a set $W$, the modularity requirement is equivalent
to postulating the existence of a rank function $k: W \rightarrow \Omega$, such that $\Omega$ is a totally ordered set.
In the presence of the well-foundedness condition above, the totally ordered set $\Omega$ happens to be a well-order,
and
we can introduce a rank function $k_\emme: \Delta \longmapsto \mathit{Ord}$ assigning an ordinal to each domain element in $W$, and  
let  $x < y$ if and only if $k_\emme(x) < k_\emme(y)$. We call $k_{\emme}(x)$ \emph{the rank of element} $x$ in $\emme$. 
Observe that, when the rank $k_{\emme}(x)$ is finite, it can be understood as the length  of a chain $x_0 < \dots < x$ from $x$
to a minimal $x_0$ (i.e. an $x_0$ s.t. for no ${x'}$, ${x'} < x_0$).   
\normalcolor

 Notice that the meaning of $\tip$ can be split into two parts: for any
$x$ of the domain $\Delta$,  $x \in (\tip(C))^I$ just in case
(i) $x \in C^I$, and (ii) there is no $y \in C^I$ such that $y < x$. In order to isolate the second part of the meaning of $\tip$, we introduce
a new modality $\bbox$. The basic idea is simply to interpret the preference
relation $<$ as an accessibility relation.  The
well-foundedness of $<$ ensures that typical elements of $C^I$ exist
whenever $C^I \diverso \vuoto$, by avoiding infinitely
descending chains of elements. 
The interpretation of $\bbox$ in $\emme$ is as follows:

\begin{definition}\label{def-box}
Given a model $\emme$, we extend the definition of $I$ with the following clause:
\begin{center}
      $ (\bbox C)^I = \{x \in \Delta \tc $  for every $y \appartiene \Delta$, if
    $y < x$ then $y \in C^I \}$
\end{center}
\end{definition}

\noindent It is easy to observe that $x$ is a typical instance  of $C$ if and only if it is an instance of $C$ and $\bbox \nott C$, that is to say:

\begin{proposition}\label{Relation between T an box}
Given a model $\emme$, given a concept $C$ and an element $x \in \Delta$, we have that
$$x \in (\tip(C))^I \ \mbox{iff} \  x \in (C \sqcap \bbox \neg C)^I$$
\end{proposition}

\noindent
Since we only use $\bbox$ to capture the meaning of $\tip$, in the
following we will always use the modality $\bbox$ followed by a negated concept,
as in $\bbox \neg C$.

In the next definition of a model satisfying a knowledge base,  we  extend
the function $I$ to individual constants;  we assign to each individual constant $a \in \mathcal{O}$ a
 domain element $a^I \in \Delta$. 



\normalcolor

\begin{definition}[Model satisfying a knowledge base]\label{Def-ModelSatTBox-ABox}
Given a $\shiqrt$ model $\emme$$=\sx\Delta, <, I\dx$, 
we say that: 
\begin{itemize}
\item a model $\emme$ satisfies an inclusion $C \sqsubseteq D$ if   $C^I \subseteq D^I$; similarly for role inclusions;
\item $\emme$ satisfies an assertion $C(a)$ if $a^I \in C^I$; 
\item $\emme$ satisfies an assertion $R(a,b)$ if $(a^I,b^I) \in R^I$.
\end{itemize}
 Given  a KB=(TBox,ABox), we say that:
 $\emme$  satisfies TBox if $\emme$ satisfies all  inclusions in TBox;
 $\emme$ satisfies ABox  if $\emme$ satisfies all  assertions in ABox;
 $\emme$ satisfies KB  (or, is a model of KB) if it satisfies both its TBox
and its ABox.
\end{definition}

As a difference with the approach in \cite{AIJ}, we do no longer assume the unique name assumption (UNA), namely we do not assume that each $a \in \mathcal{O}$ is assigned to a \emph{distinct} element $a^I \in \Delta$. 
  In $\alctmin$ \cite{AIJ}, in which we compare models that might have a different interpretation of concepts and that are not canonical, 
UNA avoids that models in which  two named individuals are mapped into the same domain element are preferred to those in which they are mapped into distinct ones. UNA is not needed here as we compare models with the same domain and the same interpretation of concepts, while assuming that models are canonical (see Definition \ref{def-canonical-model-DL}) and contain all the possible domain elements ``compatible'' with the KB. 
\normalcolor

\hide{
\noindent By a construction similar to that used in Theorem $2.3$ of \cite{FI09} for the weaker logic $\alct$, 
we can prove the following theorem. 

\begin{theorem}[Complexity of $\shiqrt$]\label{complexityalctr}
Given a $\shiqrt$ knowledge base KB=(TBox,ABox), the problem of deciding satisfiability of KB is \textsc{ExpTime}-complete.
\end{theorem}
}

The logic $\shiqrt$, as well as the underlying $\shiq$, does not enjoy the finite model property  \cite{HorrocksIGPL00}.

Given a KB,  we say that an inclusion $C_L \sqsubseteq C_R$ is entailed by KB, written KB $\modelsshiqrt C_L \sqsubseteq C_R$, if ${C_L}^I \subseteq {C_R}^I$ holds in all models $\emme=$$\sx \Delta, <, I\dx$ satisfying KB; similarly for role inclusions.
    We also say that an assertion $C_L(a)$, with $a \in \mathcal{O}$,  is entailed by KB, written KB $\modelsshiqrt C_L(a)$, if $a^I \in {C_L}^I$ holds in all models $\emme=$$\sx \Delta, <, I\dx$ satisfying KB.



 Let us now introduce the notions of rank of a $\shiq$ concept.

\begin{definition}[Rank of a concept $k_{\emme}(C_R)$]\label{definition_rank_formula}
Given a model $\emme=$$\langle \Delta, <, I \rangle$, 
we define the {\em rank} $k_{\emme}(C_R)$ of a concept $C_R$ in the model $\emme$ as $k_{\emme}(C_R)=min\{k_{\emme}(x) \tc
x \in {C_R}^I\}$. If ${C_R}^I=\vuoto$, then
$C_R$ has no rank and we write $k_{\emme}(C_R)=\infty$.
\end{definition}

\hide{\noindent It is immediate to verify that:}

\begin{proposition}\label{Truth conditions conditionals with rank}
For any $\emme=$$\langle \Delta, <, I \rangle$, we
have that $\emme$ satisfies $\tip(C) \sqsubseteq D$ if and only if $k_{\emme}(C \sqcap D) < k_{\emme}(C
\sqcap \nott D)$.
\end{proposition}

\noindent It is immediate to verify that the typicality operator $\tip$ itself  is nonmonotonic: 
$\tip(C) \sqsubseteq D$ does not imply $\tip(C \sqcap E)
\sqsubseteq D$. This nonmonotonicity of $\tip$ allows to express the properties that hold for the typical instances of a class (not only the properties that hold for all the members of the class).  However, the logic $\shiqrt$ is monotonic: what is inferred
from KB can still be inferred from any KB' with KB $\subseteq$ KB'.  This is a clear limitation in DLs. As a consequence of the monotonicity of $\shiqrt$, one cannot deal with irrelevance. For instance, if typical VIPs have more than two marriages, we would like to conclude that also typical tall VIPs have more than two marriages, since being tall is irrelevant with respect to being married. However, 
 KB$= \{\mathit{VIP} \sqsubseteq \mathit{Person}$,
  $\tip(\mathit{Person}) \sqsubseteq \ \leq 1 \ \mathit{HasMarried}.\mathit{Person}$,
  $\tip(\mathit{VIP}) \sqsubseteq \ \geq 2 \ \mathit{HasMarried}.\mathit{Person} \}$
does not entail  
   KB $\modelsshiqrt \tip(\mathit{VIP} \sqcap \mathit{Tall}) \sqsubseteq \ \geq 2 \ \mathit{HasMarried}.\mathit{Person}$, even if the property of being tall is irrelevant with respect to the number of marriages. 
Observe that we do not want to draw this conclusion in a monotonic way from $\shiqrt$, since otherwise  we would  not be able to retract it when knowing, for instance, that typical  tall  VIPs  have just one marriage (see also Example \ref{exampleActor}). Rather, we would like to obtain this conclusion in a nonmonotonic way. 
In order to obtain this nonmonotonic behavior, \normalcolor we strengthen the semantics of $\shiqrt$ by defining a minimal models mechanism  which is similar, in spirit, to circumscription. Given a KB, the idea is to: 1. define a \emph{preference relation} among $\shiqrt$ models, giving preference to the model in which domain elements have a lower rank; 
  2. restrict entailment to \emph{minimal} $\shiqrt$ models (w.r.t. the above preference relation) of KB.


\begin{definition}[Minimal models]\label{Preference between models in case of fixed valuation} 
Given $\emme = $$\langle \Delta, <, I \rangle$ and $\emme' =
\langle \Delta', <', I' \rangle$ we say that $\emme$ is preferred to
$\emme'$ \hide{with respect to the fixed interpretations minimal
semantics} ($\emme <_{\mathit{FIMS}} \emme'$) if (i) $\Delta = \Delta'$, (ii) $C^I =
C^{I'}$ for all concepts $C$, and (iii) for all $x \in \Delta$, $ k_{\emme}(x) \leq k_{\emme'}(x)$ whereas
there exists $y \in \Delta$ such that $ k_{\emme}(y) < k_{\emme'}(y)$.
Given a KB, we say that
$\emme$ is a minimal model of KB with respect to $<_{\mathit{FIMS}}$ if it is a model satisfying KB and  there is no
$\emme'$ model satisfying KB such that $\emme' <_{\mathit{FIMS}} \emme$.
\end{definition}


%
\noindent
The minimal model semantics introduced above is similar to the one introduced in \cite{AIJ} for $\alc$. However, it is worth noticing that
 the notion of minimality here is based on the minimization of the ranks of the worlds, rather then on the minimization of  formulas of a specific kind. 
Differently from \cite{AIJ}, here we only compare models in which the interpretation of concepts is the same.
In this respect, the minimal model semantics above is similar to the minimal model semantics FIMS, introduced in \cite{jelia2012}
to provide a semantic characterization to rational closure in propositional logic.
In FIMS, the interpretation of propositions in the models to be compared is fixed. In contrast, in the alternative semantic characterization VIMS, models are compared in which the interpretation of propositions may vary. 
Although fixing the interpretation of propositions (or concepts) can appear to be rather restrictive, 
for the propositional case, it has been proved  in \cite{jelia2012} that the two semantic characterizations (VIMS and FIMS) are equivalent
under suitable assumptions and, in particular, under the assumption that in FIMS canonical models are considered. 
Similarly to FIMS, here we compare models by fixing the interpretation of concepts, and we also restrict our consideration 
to canonical models, as we will do in section \ref{canonicalmodels}\footnote{
Note that our language does not provide a direct way for minimizing roles.
On the other hand, fixing roles does not appear to be very promising.
Indeed, for circumscribed KBs, it has been proved in \cite{bonattilutz} that allowing role names to be fixed
makes reasoning highly undecidabe. 
For the time being we have not studied the issue of allowing fixed roles in 
our minimal model semantics for $\shiqrt$.
}.
\normalcolor
 
\noindent Let us  define:

\begin{quote}
 	 $K_F = \{C \sqset D\in TBox: \tip \ \mbox{does not occur in $C$}\} \cup $\\
	 $\mbox{\ \ \ \ \ \ \ \ \ \ \ \ } \{R \sqset S \in TBox\}  \cup ABox$\\
	 $K_D = \{\tip(C) \sqset D\in TBox\}$,
\end{quote}
\normalcolor

\noindent so that KB $= K_F \cup K_D$.

\hide{We can prove the following:}
\begin{proposition}[Existence of minimal models]\label{modello_minimo_DL}
Let KB be a finite knowledge base, if KB is satisfiable then it has a minimal model.
\end{proposition}
\begin{proof}
Let $\emme = \la \Delta, <, I\ra$ be a model of KB, where we assume that 
$k_\emme: \Delta \longrightarrow Ord$
determines $<$ and $Ord$ is the set of ordinals. Define the relation
\begin{center}
$\emme \approx \emme'$ if $\emme'=\la \Delta', <', I'\ra$ and $\Delta = \Delta'$ and $I = I'$
\end{center}
where $<'$ is also determined by a rank $k_{\emme'}$ on ordinals. 
Define further $\mathit{Mod}_{\mathit{KB}}(\emme) = \{\emme' \mid \emme' \models \mathit{KB} \ \mbox{and} \ \emme' \approx \emme\}$.
Let us define finally $\emme_{min} = \la \Delta, <^{min}, I^{min}\ra$, where $I^{min}=I$ and $<^{min}$ is defined by the ranking, for any $x\in \Delta$:
\begin{center}
$k_{min}(x) = min \{k_{\emme'}(x) \mid \emme' \in \mathit{Mod}_{\mathit{KB}}(\emme)\}$
\end{center}
Observe that $k_{min}(x)$ is well-defined for any concept $C$  and
\begin{center}
$k_{\emme_{min}}(C) = min \{k_{min}(x) \mid x\in C^{I^{min}} \}$
\end{center}
is also well-defined (a set of ordinals has always a least element). We now show that $\emme_{min} \models$ KB. Since $I$ is the same as in $\emme$, it follows immediately that $\emme \models K_F$. 

We prove that $\emme\models K_D$. Let $\tip(C) \sqset E\in F_D$. Suppose by absurdity that $\emme_{min}\not\models \tip(C) \sqset E$, this means that
$k_{min}(C\sqcap \neg E) \leq k_{min}(C\sqcap  E)$.
Let $\emme_1\in \mathit{Mod}_{KB}(\emme)$, such that
$k_{min}(C\sqcap \neg E) = k_{\emme_1}(C\sqcap \neg E)$.
$\emme_1$ exists. Similarly, let $\emme_2 \in \mathit{Mod}_{KB}(\emme)$, such that
$k_{min}(C\sqcap  E) = k_{\emme_2}(C\sqcap  E)$.
We then have
$k_{\emme_1}(C\sqcap \neg E) = k_{min}(C\sqcap \neg E) \leq k_{min}(C\sqcap  E) =$  $k_{\emme_2}(C\sqcap  E) \leq k_{\emme_1}(C\sqcap  E)$,
as $k_{\emme_2}(C\sqcap  E)$ is minimal. Thus we get that
$k_{\emme_1}(C\sqcap \neg E) \leq k_{\emme_1}(C\sqcap  E)$
against the fact that $\emme_1$ is a model of KB.
$\hfill \bbox$
\vspace{-0.3cm} \end{proof}


The following theorem says that
reasoning in $\shiqrt$ has the same complexity as reasoning in $\shiq$, i.e. it is in \textsc{ExpTime}.
Its proof is given by providing an encoding of satisfiability in $\shiqrt$ into satisfiability $\shiq$,
which is known to be an  \textsc{ExpTime}-complete problem.

\begin{theorem}\label{encoding_shiqrt_shiq}
Satisfiability in  $\shiqrt$ is an  \textsc{ExpTime}-complete problem.
\end{theorem}
The proof can be found in  Appendix \ref{appendiceencoding}.



\vspace{-0.15cm}
\section{Rational Closure for $\shiq$} \label{Rat_Closure_DL}
\vspace{-0.25cm}
In this section, we extend to $\shiq$ the notion of rational closure proposed by Lehmann and Magidor \cite{whatdoes} for the propositional case. 
Given the typicality operator, the typicality inclusions $\tip(C) \sqsubseteq D$ (all the typical $C$'s are $D$'s) play the role of conditional assertions $C \ent D$ in \cite{whatdoes}. 
%
Here we
define the rational closure of the TBox.
In Section \ref{Section:Abox} we will discuss an extension of rational closure that also takes into account the ABox.


\normalcolor

\begin{definition}[Exceptionality of concepts and inclusions]
Let $T_B$ be a TBox and $C$ a concept. $C$ is
said to be {\em exceptional} for $T_B$ if and only if $T_B \modelsshiqrt \tip(\top) \sqsubseteq
\neg C$. A \tip-inclusion $\tip(C) \sqsubseteq D$ is exceptional for $T_B$ if $C$ is exceptional for $T_B$. The set of \tip-inclusions of $T_B$ which are exceptional in $T_B$ will be denoted
as $\mathcal{E}$$(T_B)$.
\end{definition}

\noindent Given a DL  KB=(TBox,ABox),
it is possible to define a sequence of non increasing subsets of
TBox $E_0 \supseteq E_1, E_1 \supseteq E_2, \dots$ by letting $E_0 =\mbox{TBox}$ and, for
$i>0$, $E_i=\mathcal{E}$$(E_{i-1}) \unione \{ C \sqsubseteq D \in \mbox{TBox}$ s.t. $\tip$ does not occurr in $C\}$.
Observe that, being KB finite, there is
an $n\geq 0$ such that, for all $m> n, E_m = E_n$ or $E_m =\emptyset$.
Observe also that the definition of the $E_i$'s is the same as the definition of the $C_i$'s in Lehmann and Magidor's rational closure \cite{KrausLehmannMagidor:90},
except for that here, at each step, we also add all the ``strict'' inclusions $C \sqsubseteq D$ (where $\tip$ does not occur in $C$).

\begin{definition}[Rank of a concept]\label{Def:Rank of a formula_DL} A concept $C$ has {\em rank} $i$ (denoted by $\rf(C)=i$) for KB=(TBox,ABox),
iff $i$ is the least natural number for which $C$ is
not exceptional for $E_{i}$. {If $C$ is exceptional for all
$E_{i}$ then $\rf(C)=\infty$, and we say that $C$ has no rank.}
\end{definition}

 \noindent The notion of rank of a formula allows to define
the rational closure of the TBox of a KB. 
Let $\models_{\shiq}$ be the entailment in $\shiq$.
In the following definition, by 
KB $\models_{\shiq} F$ we mean $K_F \models_{\shiq} F$,
where $K_F$ does not include the defeasible inclusions in KB.

\begin{definition}[Rational closure of TBox] \label{def:rational closureDL}
Let KB=(TBox,ABox) be a DL knowledge base. We define, $\overline{\mathit{TBox}}$, the
{\em rational closure  of TBox}, as
    $\mbox{$\overline{\mathit{TBox}}$}=\{\tip(C) \sqsubseteq D \tc \mbox{either} \ \rf(C) < \rf(C \sqcap \nott D)$ $\mbox{or} \ \rf(C)=\infty\} \ \unione \
    \{C \sqsubseteq D \tc \ \mbox{KB} \ \models_{\shiq} C \sqsubseteq D\}$,
where $C$ and $D$ are arbitrary $\shiq$ concepts.
\end{definition}

\noindent 
Observe that, apart form the addition of strict inclusions, the above definition of rational closure is the same as the one by Lehmann and Magidor in \cite{whatdoes}. \normalcolor
The rational closure of TBox is a nonmonotonic strengthening of $\shiqrt$. For instance, it allows to deal with irrelevance, as the following example shows. 

\begin{example}  \label{exampleActor} Let TBox = $\{\tip(\mathit{Actor}) \sqsubseteq \mathit{Charming}\}$.  It can be verified that
  $\tip(\mathit{Actor} \sqcap \mathit{Comic}) \sqsubseteq \mathit{Charming} \in \overline{\mathit{TBox}}$. This is a nonmonotonic inference that does no longer follow if we discover that indeed comic actors  are not charming (and in this respect are untypical actors): indeed given TBox'= TBox $\cup \ \{ \tip(\mathit{Actor} \sqcap \mathit{Comic}) \sqsubseteq \neg \mathit{Charming} \}$, we have that $\tip(\mathit{Actor} \sqcap \mathit{Comic}) \sqsubseteq  \mathit{Charming} \not\in \overline{\mathit{TBox'}}$.

\noindent
Furthermore, as for the propositional case, rational closure is closed under rational monotonicity \cite{KrausLehmannMagidor:90}: from $\tip(\mathit{Actor}) \sqsubseteq \mathit{Charming} \in \overline{\mathit{TBox}}$ and $\tip(\mathit{Actor}) \sqsubseteq  \mathit{Bold} \not\in \overline{\mathit{TBox}}$  it follows that  $\tip(\mathit{Actor} \sqcap \nott  \mathit{Bold} ) \sqsubseteq \mathit{Charming} \in \overline{\mathit{TBox}}$.

 \end{example}



Although the rational closure $\overline{\mathit{TBox}}$ is an infinite set,  its definition is based on the construction of
a finite sequence $E_0, E_1, \ldots, E_n$ of subsets of TBox, and the problem of verifying  that an inclusion 
$\tip(C) \sqsubseteq  D \in \overline{\mathit{TBox}}$ is in \textsc{ExpTime}.
To prove this result we need to introduce some propositions. 

First of all, let us remember that rational entailment is equivalent to preferential entailment for a knowledge base only containing positive  non-monotonic implications $A \ent B$ (see \cite{whatdoes}). 
The same holds in preferential description logics with typicality.
Let $\shiqpt$ be the logic that we obtain when we remove the requirement of modularity in the definition of $\shiqrt$.
In this logic the typicality operator has a preferential semantics  \cite{KrausLehmannMagidor:90}, based on the preferential models of  {\Pe}  rather then on the ranked models \cite{whatdoes}. 
An extension of $\alc$ with typicality based on preferential logic {\Pe} has been studied in \cite{FI09}. 
As a TBox of a KB in $\shiqrt$ is a set of strict inclusions and defeasible inclusions (i.e., positive non-monotonic implications), it can be proved that:

\begin{proposition}\label{ultimanicola}
Given a $KB$ \textit{with empty ABox}, and an inclusion $E \sqsubseteq D$ we have 
$$KB \models_{\shiqrt} E \sqsubseteq D \ \mbox{iff} \ KB \models_{\shiqpt} E \sqsubseteq D$$
\end{proposition}
\begin{proof}
(sketch) The (if) direction is trivial, thus we consider the (only if) one.  Suppose that $KB \not \models_{\shiqpt}  E \sqsubseteq D$, let $\emme = \sx \Delta, <, I \dx$ a preferential model of $KB$, where $<$ is transitive, irreflexive, and well-founded,  
which falsifies $E \sqsubseteq D$. Then for some element $x\in E$ and $x\not\in  D$.  Define first a model $\emme_1 = \sx \WW, <_1, I \dx$, where the relation $<_1$ is defined as follows:
$$<_1 = < \ \cup \ \{(u, v) \mid (u  = x \lor u < x) \land v \not = x \land v\not < x\}$$
It can be proved that:
\begin{enumerate}
	\item $<_1$ is transitive and irreflexive
	\item $<_1$ is well-founded
	\item if $u < v$ then $u <_1 v$
	\item if $u <_1 x$ then $u < x$.
\end{enumerate}
We can show that $\emme_1$ is a model of $KB$. This is obvious for inclusions that do not involve $\tip$, as the interpretation $I$ is the same. Given an inclusion $\tip(G) \sqsubseteq F\in KB$, if it holds in $\emme$ then it holds also in $\emme_1$ as $Min^{\emme_1}_{<_1}(G) \subseteq Min^{\emme}_{<}(G)$.
Moreover $\emme_1$ falsifies $E \sqsubseteq D$ by $x$, in particular (the only interesting case) when  $E= \tip(C)$. To this regard, we know that $x\not\in  D^{\emme_1}$, suppose by absurd that 
$x\not\in (\tip(C))^{\emme_1}$, since $x\in  (\tip(C))^{\emme}$, we have that $x\in  C^\emme = C^{\emme_1}$, thus there must be a $y <_1 x$ with  $y\in C^{\emme_1} = C^\emme$. But then by 4
$y < x$ and we get a contradiction. Thus  $x\in (\tip(C))^{\emme_1}$ and $x\not\in D^{\emme_1}$, that is $x$ falsifies  $E \sqsubseteq D$ in $\emme_1$.

Observe that  $<_1$ in model $\emme_1$ satisfies:
$$(*) \ \forall z\not= x \ (z <_1 x \lor x <_1 z)$$
As a next step we define a \textbf{\emph{modular}} model $\emme_2 = \sx \WW, <_2, I \dx$, where the relation $<_2$ is defined as follows. 
Considering $\emme_1$ where $<_1$ is well-founded, we can define by recursion  the following function $k$ from $\emme$ to ordinals:
\begin{itemize}
\item $k(u) = 0$ if $u$ is minimal in $\emme_1$
\item $k(u) = max \{k(y) \mid y <_1 u\}+1$ if the set $\{y \mid y <_1 u\}$ is finite
\item $k(u) = sup \{k(y) \mid y <_1 u\}$ if the set $\{y \mid y <_1 u\}$ is infinite.
\end{itemize}
Observe that if $u <_1 v$ then $k(u) < k(v)$. We now define:
$$u <_2 v \ \mbox{iff} \ k(u) < k(v)$$
Notice that $<_2$ is clearly transitive, modular, and well-founded; moreover $u <_1 v$ implies $u <_2 v$. We can prove as before that $\emme_2$ is a model of $KB$ and that it falsifies $E \sqsubseteq D$ by $x$. For the latter, we consider again the only interesting case when  $E= \tip(C)$. Suppose by absurd that $x\not\in (\tip(C))^{\emme_2}$, since $x\in  (\tip(C))^{\emme_1}$, we have that $x\in  C^{\emme_2} = C^{\emme_1}$, thus there must be a $y <_2 x$ with  $y\in C^{\emme_2} = C^{\emme_1}$. But $y <_2 x$ means that $k(y) < k(x)$. We can conclude  that it must be also $y <_1 x$, otherwise by (*) we would have $x <_1 y$  which entails $k(x) < k(y)$, a contradiction. We have shown that $y <_1 x$, thus $x\not\in (\tip(C))^{\emme_1}$ a contradiction. Therefore  
$x\in (\tip(C))^{\emme_2}$ and $x\not\in  D^{\emme_2}$, that is $x$ falsifies  $E \sqsubseteq D$ in $\emme_2$. 
 We have shown that $KB \not \models_{\shiqrt} E \sqsubseteq D$. 
 $\hfill \bbox$
\end{proof}

The proof above also extends to a KB with a non-empty ABox, but it must not contain positive typicality assertions on individuals.

\begin{proposition} \label{relazione_shiq_shiqrt}
Let KB=(TBox,$\emptyset$) be a knowledge base with empty ABox.
 $KB \models_{\shiqrt} C_L \sqsubseteq C_R$ iff 
$KB' \models_{\shiq} C'_L \sqsubseteq C'_R$,
where  $KB'$, $C'_L$ and $C'_R$ are polynomial encodings in $\shiq$ of KB, $C_L$ and $C_R$, respectively.
\end{proposition}
\begin{proof} 
By Proposition \ref{ultimanicola}, we have that

\begin{center}
 $KB \models_{\shiqrt} C_L \sqsubseteq C_R$ iff $KB \models_{\shiqpt} C_L \sqsubseteq C_R$
 \end{center}
where $C_L \sqsubseteq C_R$ is any (strict or defeasible) inclusion in $\shiqrt$. 

 To prove the thesis it suffices to show that 
for all inclusions $ C_L \sqsubseteq C_R$ in $\shiqrt$:
\begin{center}
$KB \models_{\shiqpt} C_L \sqsubseteq C_R$ iff  $KB' \models_{\shiq} {C'_L} \sqsubseteq {C'_R}$
 \end{center}
for some polynomial encoding  $KB'$, $C'_L$ and $C'_R$ in $\shiq$.\normalcolor 

The idea, on which the encoding is based, exploits the definition of the typicality operator $\tip$
introduced in \cite{FI09}, in terms of a G\"odel-L\"ob modality $\Box$ as follows: 
$\tip(C)$ is defined as $C \sqcap \Box \neg C$
where the accessibility relation of the modality  $\Box$ is the preference relation $<$ in preferential models.

We define the encoding KB'=(TBox', ABox') of KB in $\shiq$ as follows.
First, ABox'=$\emptyset$.

For each $A \sqsubseteq B \in$ TBox, not containing $\tip$, we introduce  
$A \sqsubseteq B $ in TBox'.

For each $\tip(A)$ occurring in the TBox, we introduce a new atomic concept  $\Box_{\neg A}$
and, for each inclusion $\tip(A) \sqsubseteq B \in$ TBox, we add to TBox'
 the inclusion 
 $$ A \sqcap \Box_{\neg A} \sqsubseteq B$$
 Furthermore, to capture the properties of the $\Box$ modality, a new 
role $R$ is introduced to represent the relation $<$ in preferential models,
 and the following inclusions are introduced in TBox':
\begin{center}
$ \Box_{\neg A} \sqsubseteq  \forall R. (\neg A \sqcap  \Box_{\neg A})$\\
$ \neg \Box_{\neg A} \sqsubseteq  \exists R. (A \sqcap  \Box_{\neg A})$
\end{center}
The first inclusion accounts for the transitivity of $<$.
The second inclusion accounts for the smoothness (see \cite{whatdoes,FI09}): the fact that if an element is not a typical $A$ element then there must be a typical $A$ element preferred to it. 

For the encoding of the inclusion $C_L \sqsubseteq C_R$:
if $C_L \sqsubseteq C_R$ is a strict inclusion in $\shiqrt$, then  $C'_L= C_L$ and  $C'_R=C_R$;
if $C_L \sqsubseteq C_R$ is a defeasible inclusion in $\shiqrt$, i.e. $C_L=\tip(A)$,
then, we define $C'_L= A \sqcap \Box_{\neg A} $ and  $C'_R=C_R$.

It is clear that the size of KB' is polynomial in the size of the KB (and the same holds for $C'_L$ and $C'_R$,
assuming the size of $C_L$ and $C_R$ polynomial in the size of the KB).
Given the above encoding, we can prove that:
\begin{center}
$KB \models_{\shiqpt} C_L \sqsubseteq C_R$ iff  $KB' \models_{\shiq} {C'_L} \sqsubseteq {C'_R}$
 \end{center}
 
$(If)$ 
By contraposition, let us assume that $KB \not \models_{\shiqpt} C_L \sqsubseteq C_R$.
We want to prove that $KB' \not \models_{\shiq} {C'_L} \sqsubseteq {C'_R}$.
From the hypothesis, there is a preferential model $\emme =(\Delta,<, I)$ satisfying KB such that  
for some element $x\in \Delta$, $x \in (C_L)^I$ and $x \in (\neg C_R)^I$.
We build a $\shiq$ model $\emme'=(\Delta',I')$ satisfying KB' as follows:
\begin{quote}

$\Delta'=\Delta$;

$C^{I'} =C^I$, for all concepts $C$ in the language of $\shiq$;

$R^I = R^{I'} $, for all roles $R$;

$(x,y)\in R^{I'}$ if and only if $y<x$ in the model $\emme $.

\end{quote}
\noindent
By construction it follows that $\tip(A)^I = (A \sqcap \Box_{\neg A})^{I'}$.
Also, it can be easily verified that
$\emme$ satisfies all the inclusions in KB' and that
$x \in (C'_L)^{I'}$ and $x \in (\neg C'_R)^{I'}$.
Hence $KB' \not \models_{\shiq} {C'_L} \sqsubseteq {C'_R}$.

$(Only\: if)$
By contraposition, let us assume that $KB' \not \models_{\shiq} {C'_L} \sqsubseteq {C'_R}$.
We want to prove that  $KB \not \models_{\shiqpt} C_L \sqsubseteq C_R$.
From the hypothesis, we know there is a model $\emme'=(\Delta',I')$ satisfying KB',
such that  $x \in (C'_L)^{I'}$ and $x \in (\neg C'_R)^{I'}$.
We build a model $\emme =(\Delta,<, I)$ satisfying KB such that  some element of 
$\emme $ does not satisfy the inclusion $C_L \sqsubseteq C_R$. We let:

\begin{quote}
$\Delta=\Delta'$;

$C^I = C^{I'} $, for all concepts $C$ in the language of $\shiq$;

$R^I = R^{I'} $, for all roles $R$;

$y<x$  if and only if  $(x,y)\in (R^{I'})^*$ (the transitive closure of $R^{I'}$).
\end{quote}
\noindent
By construction, it is easy to show that $\tip(A)^I = (A \sqcap \Box_{\neg A})^{I'}$
and we can easily verify that  $\emme$ satisfies all the inclusions in KB and that
$x \in (C_L)^I$ and $x \in (\neg C_R)^I$.

The relation $<$ is transitive, as it is defined as the transitive closure of $R$,
but $<$ is not guaranteed to be well-founded.
However, we can modify the relation $<$ in $\emme$ to make it well-founded, by shortening the descending chains.

For any $y\in \Delta$, we let $\Box_y= \{ \Box C\; \mid \; y\in (\Box C)^I \}$. 
Observe that for the elements $x_i$ in a descending chain $\ldots, x_{i-1}, x_i, x_{i+1}, \ldots$, the set $\Box_{x_i}$ is monotonically increasing
(i.e., $\Box_{x_i} \subseteq \Box_{x_{i+1}}$).

We define a new model $\emme''=(\Delta,<'', I)$ by changing the preference relation $<$ in $\emme$ to $<''$ as follows:
\begin{quote}
$y<''x$  iff ($y<x$ and $\Box_x \subset \Box_y$) or \\
$\mbox{\ \ \ \ }$ \hspace{1cm} ($y<x$ and $\Box_x = \Box_y$ and $\forall w\in \Delta$ such that $x<w$, $\Box_w \subset \Box_x$) 

\end{quote}
In essence, for a pair of elements $(x,y)$ such that $y<x$ but $x$ and $y$ are instances of  exactly the same boxed concepts
($\Box_x = \Box_y$) and $x$ is not the first element in the descending chain which is instance of all the boxed concepts in $\Box_x$,
we do not include the pair $(x,y)$ in $<''$ (so that $x$ and $y$ will not be comparable in the pre-order $<''$).   
The relation $<''$ is transitive and well-founded. 
$\emme''$ can be shown to be a model of KB, and $x$ to be an instance of $C_L$ but not of $C_R$.
Hence, $KB \not \models_{\shiqpt} C_L \sqsubseteq C_R$.
\normalcolor $\hfill \bbox$
\end{proof}

\normalcolor
\begin{theorem}[Complexity of rational closure over TBox]
Given a TBox, the problem of deciding whether
$\tip(C) \sqsubseteq D \in \overline{\mathit{TBox}}$ is in \textsc{ExpTime}.
\end{theorem}

\begin{proof}
Checking if $\tip(C) \sqsubseteq D \in \overline{\mathit{TBox}}$ can be done by
computing the finite sequence  $E_0, E_1, \ldots, E_n$ of non increasing subsets of TBox inclusions
in the construction of the rational closure.
Note that the number $n$ of the $E_i$ is  $O(|KB|)$,
where $| KB|$ is the size of the knowledge base KB.
Computing each $E_i ={\cal E}(E_{i-1})$, requires to check, for all concepts $A$
occurring on the left hand side of an inclusion in the TBox, whether $E_{i-1} \modelsshiqrt \tip(\top) \sqsubseteq \neg A$. 
Regarding $E_{i-1}$ as a knowledge base with empty ABox,
by Proposition \ref{relazione_shiq_shiqrt} it is enough to check that
${E'_{i-1}} \models_{\shiq} \top \sqcup \Box_{\neg \top} \sqsubseteq \neg A$, which requires an exponential time in the size of ${E'_{i-1}}$ (and hence in the size of KB).
If not already checked, the exceptionality of $C$ and of $C\sqcap \neg D$ have to be checked for each $E_i$,
to determine the ranks of $C$ and of $C \ \sqcap \ \neg D$ (which also can be computed in $\shiq$ and requires an exponential time in the size of KB).
Hence, verifying if $\tip(C) \sqsubseteq D \in \overline{\mathit{TBox}}$ is in \textsc{ExpTime}. \hfill $\Box$
\end{proof}
The above proof provides an \textsc{ExpTime} complexity upper bound for computing the rational closure over a TBox in $\shiq$
and shows that the rational closure of a TBox can be computed simply using the entailment in $\shiq$.


\section{Infinite Minimal Models with finite ranks} \label{infinite_models}

In the following we provide a characterization of minimal models of a KB in terms of their rank:  intuitively minimal models are exactly those ones where  each domain element has rank $0$ if it satisfies all defeasible inclusions, and otherwise has the smallest  rank greater than the rank of any concept $C$ occurring in a defeasible inclusion $\tip(C) \sqset D$ of the KB  falsified by the element. 
Exploiting this intuitive characterization of minimal models, we are able to show that, for a finite KB, minimal models have always a \emph{finite} ranking function, no matter whether they have a finite domain or not. 
This result allows us to provide a semantic characterization of rational closure of the previous section to logics, like $\shiq$, that do not have the finite model property. 

Given a model $\emme = \langle \Delta, <, I \rangle$, let us define the set $S^{\emme}_x$ of defeasible inclusions falsified by a domain element $x\in \Delta$, as
$S^{\emme}_x = \{\tip(C) \sqset D\in K_D \mid x\in (C \sqcap \neg D)^I\}\}$.

\begin{proposition}\label{cara1}
Let $\emme = \langle \Delta, <, I \rangle$ be a model of KB and $x\in \Delta$, then:
 (a) if $k_{\emme}(x) = 0$ then $S^{\emme}_x = \emptyset$;
	 (b) if $S^{\emme}_x \not= \emptyset$ then $k_{\emme}(x) > k_{\emme}(C)$ for every $C$ such that, for some $D$, $\tip(C) \sqset D \in S^{\emme}_x$.
\end{proposition}
\begin{proof}
Observe that (a) follows from (b). Let us prove (b). Suppose for a contradiction that (b) is false, so that $S^{\emme}_x \not= \emptyset$ and for some $C$ such that, for some $D$, $\tip(C) \sqset D \in S^{\emme}_x$, we have $k_{\emme}(x) \leq k_{\emme}(C)$.  We have also that $x\in (C \sqcap \neg D)^I$.  But  $\emme \models$ KB, in particular $\emme \models \tip(C) \sqset D$, thus it must be $x\not\in (\tip (C))^I$, but $x\in C^I$,  so that we get that $k_{\emme}(x) > k_{\emme}(C)$ a contradiction. 
 \vspace{-0.2cm}
$\hfill \Box$ \end{proof}

\begin{proposition}\label{cara2}
Let KB $= K_F \cup K_D$ and $\emme = \langle \Delta, <, I \rangle$ be a model of $K_F$; suppose that for any $x\in \Delta$ it holds:
\begin{itemize}
	\item (a) if $k_{\emme}(x) = 0$ then $S^{\emme}_x = \emptyset$
	\item (b) if $S^{\emme}_x \not= \emptyset$ then $k_{\emme}(x) > k_{\emme}(C)$ for every $C$ such that, for some $D$, $\tip(C) \sqset D \in S^{\emme}_x$.
\end{itemize}
then $\emme\models$ KB.
\end{proposition}
\begin{proof}
Let $\tip(C) \sqset D\in K_D$, suppose that for some $x\in C$, it holds $x\in (\tip(C))^I - D^I$, then $\tip(C) \sqset D \in S^{\emme}_x$. By hypothesis, we have $k_{\emme}(x) > k_{\emme}(C)$, against the fact that $x\in \tip(C)$.
$\hfill \Box$ \end{proof}

\begin{proposition}\label{cara3}
Let KB $= K_F \cup K_D$ and $\emme = \langle \Delta, <, I \rangle$  a \emph{ minimal model} of KB, for every $x\in \Delta$, it holds:
\begin{itemize}
	\item (a) if $S^{\emme}_x = \emptyset$ then $k_{\emme}(x) = 0$ 
	\item (b) if $S^{\emme}_x \not= \emptyset$ then $k_{\emme}(x) = 1+ max\{k_{\emme}(C) s.t. \tip(C) \sqset D \in S^{\emme}_x\}$.
\end{itemize}
\end{proposition}
\begin{proof}
Let $\emme = \langle \Delta, <, I \rangle$ be a minimal model of KB. Define another model $\emme' = \langle \Delta, <', I \rangle$, where $<'$ is determined by a ranking function $k_{\emme'}$ as follows:
\begin{itemize}
	\item $k_{\emme'}(x) = 0$ if $S^{\emme}_x=\emptyset$,
	\item $k_{\emme'}(x) = 1+ max\{k_{\emme}(C) \mid \tip(C) \sqset D \in S^{\emme}_x\}$ if $S^{\emme}_x\not=\emptyset$.
\end{itemize}
It is easy to see that (i) for every $x$ $k_{\emme'}(x) \leq k_{\emme}(x)$.
Indeed, if $S^{\emme}_x =\emptyset$ then it is obvious; if $S^{\emme}_x \not=\emptyset$, then
$k_{\emme'}(x) = 1+ max\{k_{\emme}(C) \mid \tip(C) \sqset D \in S^{\emme}_x\} \leq k_{\emme}(x)$
by Proposition \ref{cara1}. 
It equally follows  that
(ii) for every concept $C$, $k_{\emme'}(C) \leq k_{\emme}(C)$. To see this: let $z\in C^I$ such that $k_{\emme}(z) = k_{\emme}(C)$, either $k_{\emme'}(C) = k_{\emme'}(z) \leq k_{\emme}(z)$ and we are done, or there exists $y\in C^I$, such that $k_{\emme'}(C) = k_{\emme'}(y) < k_{\emme'}(z) \leq k_{\emme}(z)$.

Observe that $S^{\emme}_x = S^{\emme'}_x$, since the evaluation function $I$ is the same in the two models. By definition of $\emme'$, we have $\emme' \models K_F$; moreover by (i) and (ii) it follows that: 

 (iii) if $k_{\emme'}(x) = 0$ then $S^{\emme'}_x = \emptyset$.

 (iv) if $S^{\emme'}_x \not= \emptyset$:
$k_{\emme'}(x)  =  1+ max\{k_{\emme}(C) \mid \tip(C) \sqset D \in S^{\emme}_x\}
 \geq  1+ max\{k_{\emme'}(C) \mid \tip(C) \sqset D \in S^{\emme'}_x\}$, that is $k_{\emme'}(x) > k_{\emme'}(C)$ for every $C$ such that for some $D$, $\tip(C) \sqset D \in S^{\emme'}_x$.

By  Proposition  \ref{cara2} we obtain that $\emme'\models$ KB; but by (i) $k_{\emme'}(x) \leq k_{\emme}(x)$ and by hypothesis $\emme$ is minimal. 
Thus it must be that for every $x\in \Delta$,  $k_{\emme'}(x) = k_{\emme}(x)$ (whence $k_{\emme'}(C)= k_{\emme}(C)$) which entails that $\emme$ satisfies (a) and (b) in the statement of the theorem. 
$\hfill \Box$ \end{proof}

Also the opposite direction holds:

\begin{proposition}\label{cara4}
Let KB $= K_F \cup K_D$, let $\emme = \langle \Delta, <, I \rangle$  be a model of $K_F$,  suppose that for every $x\in \Delta$, it holds:
\begin{itemize}
	\item (a)  $S^{\emme}_x = \emptyset$ iff $k_{\emme}(x) = 0$ 
	\item (b) if $S^{\emme}_x \not= \emptyset$ then $k_{\emme}(x) = 1+ max\{k_{\emme}(C) \mid \tip(C) \sqset D \in S^{\emme}_x\}$.
\end{itemize}
then $\emme$ is  a minimal model of KB.
\end{proposition}
\begin{proof}
In light of previous Propositions \ref{cara1} and \ref{cara2}, it is sufficient to show that $\emme$ is minimal. To this aim, let $\emme' = \langle \Delta, <', I \rangle$, with associated ranking function $k_{\emme'}$, be another model of KB, we show that for every $x\in \Delta$, it holds $k_{\emme}(x) \leq k_{\emme'}(x)$. 
We proceed by induction on $k_{\emme'}(x)$. If  $S^{\emme}_x =  S^{\emme'}_x = \emptyset$, we have that $k_{\emme}(x) = 0 \leq k_{\emme'}(x)$ (no need of induction). If  $S^{\emme}_x =  S^{\emme'}_x \not= \emptyset$, then  since $\emme'\models$ KB, by  Proposition \ref{cara1}: $k_{\emme'}(x) \geq 1 + max\{k_{\emme'}(C) \mid \tip(C) \sqset D \in S^{\emme'}_x\}$.
Let $S^{\emme'}_x = S^{\emme }_x= \{\tip(C_1)\sqset D_1 ,\ldots, \tip(C_u)\sqset D_u\}$. 
For $i = 1,\ldots, u$  let $k_{\emme'}(C_i) =  k_{\emme'}(y_i)$ for some $y_i\in \Delta$. 
Observe that  $k_{\emme'}(y_i) < k_{\emme'}(x)$, thus by induction hypothesis $k_{\emme}(y_i) \leq k_{\emme'}(y_i)$, for $i = 1,\ldots, u$.
But then $k_{\emme}(C_i) \leq k_{\emme}(y_i)$, so that we finally get:
\begin{footnotesize}
	\begin{eqnarray*}
k_{\emme'}(x) & \geq & 1+ max\{k_{\emme}(C) \mid \tip(C) \sqset D \in S^{\emme'}_x\}\\
& = & 1+ max\{k_{\emme'}(C_1),\ldots, k_{\emme'}(C_u)\}\\
& = & 1+ max\{k_{\emme'}(y_1),\ldots, k_{\emme'}(y_u)\}\\
& \geq & 1+ max\{k_{\emme}(y_1),\ldots, k_{\emme}(y_u)\}\\
& \geq & 1+ max\{k_{\emme}(C_1),\ldots, k_{\emme}(C_u)\}\\
& = & 1+ max\{k_{\emme}(C) \mid \tip(C) \sqset D \in S^{\emme}_x\} \\ 
& = & k_{\emme'}(x) 
\end{eqnarray*}
\end{footnotesize}
$\hfill \Box$ \end{proof}

\noindent Putting Propositions \ref{cara3} and \ref{cara4} together, we obtain the following theorem which provides a characterization of minimal models.

\begin{theorem}\label{teocara} Let KB $= K_F \cup K_D$, and let $\emme = \langle \Delta, <, I \rangle$  be a model of $K_F$. The following are equivalent:
\begin{itemize}
	\item $\emme$ is a minimal model of KB
	\item For every $x\in \Delta$ it holds:
 (a)  $S^{\emme}_x = \emptyset$ iff $k_{\emme}(x) = 0$ 
	 (b) if $S^{\emme}_x \not= \emptyset$ then $k_{\emme}(x) = 1+ max\{k_{\emme}(C) \mid \tip(C) \sqset D \in S^{\emme}_x\}$.
\end{itemize}
\end{theorem}



\noindent The following proposition shows that in any minimal model the \emph{rank} of each domain element is finite.
 
\begin{proposition}\label{finiti nicola}
Let KB $= K_F \cup K_D$ and $\emme = \langle \Delta, <, I \rangle$  a minimal model of KB, for every $x\in \Delta$, $k_{\emme}(x)$  is a finite ordinal ($k_{\emme}(x) < \omega$).
\end{proposition}
\begin{proof}
Let $k_{\emme}(x)= \alpha$, we proceed by induction on $\alpha$. If $S^{\emme}_x = \emptyset$, then by  Proposition \ref{cara3} $\alpha = 0$ and we are done (no need of induction). Otherwise if $S^{\emme}_x\not = \emptyset$, by Proposition \ref{cara3}, we have that
$k_{\emme}(x)= \alpha = 1+ max\{k_{\emme}(C) \mid \tip(C) \sqset D \in S^{\emme}_x\}$.
Let $S^{\emme}_x = \{\tip(C_1)\sqset D_1 ,\ldots, \tip(C_u)\sqset D_u\}$. For $i = 1,\ldots, u$  let $k_{\emme}(C_i) = \beta_i = k_{\emme}(y_i)$ for some $y_i\in \Delta$.
So that we have
$k_{\emme}(x)= \alpha = 1+ max\{\beta_1,\ldots, \beta_u\}$.
Since $k_{\emme}(y_i) = \beta_i< \alpha$, by induction hypothesis we have that $\beta_i < \omega$, thus also $\alpha < \omega$. 
$\hfill \Box$ \end{proof}

The previous proposition is essential for establishing a correspondence between the minimal model semantics of a KB and its rational closure. From now on, we can assume that the ranking function assigns to each domain element in $\Delta$ a natural number, i.e. that $k_\emme: \Delta \longrightarrow  \mathbb{N}$.



\section{A Minimal Model Semantics for Rational Closure in $\shiq$} \label{canonicalmodels}
 In previous sections  we have extended to $\shiq$ the syntactic notion of rational closure introduced in \cite{whatdoes} for propositional logic.
To provide a semantic characterization of this notion, we define a special class of minimal models, exploiting the fact that, by Proposition \ref{finiti nicola}, in all minimal $\shiqrt$ models the \emph{rank} of each domain element is always finite.
First of all, we can observe that the minimal model semantics in Definition \ref{Preference between models in case of fixed valuation}  as it is cannot capture the rational closure of a TBox.
 
 Consider the following KB=(TBox,$\vuoto$), where TBox contains: 
 \begin{quote}
  $\mathit{VIP}  \sqsubseteq \mathit{Person}$,\\
  $\tip(\mathit{Person}) \sqsubseteq \ \leq 1 \ \mathit{HasMarried}.\mathit{Person}$,\\
  $\tip(\mathit{VIP})$ $ \sqsubseteq \ \geq 2 \ \mathit{HasMarried}.$ $\mathit{Person}$.
 \end{quote}
  We observe that 
$\tip(\mathit{VIP} \sqcap \mathit{Tall})  \sqsubseteq \ \geq 2 \ \mathit{HasMarried}.\mathit{Person}$ does not hold in all minimal $\shiqrt$ models of KB w.r.t. Definition \ref{Preference between models in case of fixed valuation}.
 Indeed there can be a model $\emme=\sx\Delta, <, I\dx$ in which
$\Delta = \{x,y,z\}$, $\mathit{VIP}^I=\{x, y\}$, $\mathit{Person}^I=\{x, y, z\}$,
$(\leq 1 \ \mathit{HasMarried}.\mathit{Person})^I=\{x, z\}$, $(\geq 2 \ \mathit{HasMarried}.\mathit{Person})^I=\{y\}$, $\mathit{Tall}^I=\{x\}$, and $z < y < x$.
$\emme$ is a model of KB, and  it is minimal. 
Also,  $x$ is a typical
tallVIP in $\emme$ (since there is no other tall VIP preferred to
him) and has no more than one spouse, therefore $\tip(\mathit{VIP} \sqcap \mathit{Tall}) \sqsubseteq \ \geq 2 \ \mathit{HasMarried}.\mathit{Person}$ does not hold in $\emme$. 
On the contrary, it can be  verified that
$\tip(\mathit{VIP} \sqcap \mathit{Tall}) \sqsubseteq \ \geq 2 \ \mathit{HasMarried}.\mathit{Person} \in \overline{\mathit{TBox}}$.

 Things change if we
consider the minimal models semantics applied to models that contain a domain element for {\em each
combination of concepts consistent with KB}.
We call these models {\em canonical models}. Therefore,
in order to semantically characterize the rational closure of a $\shiqrt$   KB, we restrict our attention to \emph{minimal canonical models}.
First, we define $\lingconc$ as the set of all the concepts (and subconcepts) not containing $\tip$,
which occur in KB or in the query $F$, together with their complements.

In order to define canonical models, we consider all the sets of concepts $\{C_1, C_2, \dots,$ $ C_n\} \subseteq \lingconc$ that are {\em consistent with KB}, i.e., s.t. KB $\not\models_{\shiqrt} C_1 \sqcap C_2 \sqcap \dots \sqcap C_n \sqsubseteq \bot$.

\begin{definition}[Canonical model with respect to $\lingconc$]\label{def-canonical-model-DL}
Given KB=(TBox,ABox) and a query $F$, a  model $\emme=$$\sx \Delta, <, I \dx$ satisfying KB is 
{\em canonical with respect to $\lingconc$} if it contains at least a domain element $x \in \Delta$ s.t.
$x \in (C_1 \sqcap C_2 \sqcap \dots \sqcap C_n)^I$, for each set of concepts
$\{C_1, C_2, \dots, C_n\} \subseteq \lingconc$ that is consistent with KB. \end{definition}

\noindent Next we define the notion of minimal canonical model.
\begin{definition}[Minimal canonical models (w.r.t. $\lingconc$)]\label{Preference between models wrt TBox}
$\emme$ is a minimal canonical model of KB
if it satisfies KB, it is  minimal (with respect to Definition \ref{Preference between models in case of fixed
valuation}) and it is canonical (as defined in Definition \ref{def-canonical-model-DL}).

\end{definition}




\begin{proposition}[Existence of minimal canonical models]\label{modello_minimo_canonico_DL}
Let KB be a finite knowledge base, if KB is satisfiable then it has a minimal canonical model.
\end{proposition}

\begin{proof}
Let $\emme = \la \Delta, <, I\ra$ be a minimal model of KB (which exists by Proposition \ref{modello_minimo_DL}), and let
$\{C_1, C_2, \dots, C_n\} \subseteq \lingconc$ any subset of $\lingconc$ consistent with KB.

We show that we can expand $\emme$ in order to obtain a model of KB that contains an instance of $C_1 \sqcap C_2 \sqcap \dots \sqcap C_n$. By repeating the same construction for all  maximal subsets $\{C_1, C_2, \dots, C_n\}$ of $\lingconc$, we eventually obtain a canonical model of KB.

For each $\{C_1, C_2, \dots, C_n\}$ consistent with KB, it holds that  KB $\not\models_{\shiqrt} C_1 \sqcap C_2 \sqcap \dots \sqcap C_n \sqsubseteq \bot$, i.e. there is a model $\emme' = \la \Delta', <', I' \ra$ of KB that contains an instance of $\{C_1, C_2, \dots, C_n\}$.

Let $\emme^{'*}$ be the union of $\emme$ and $\emme'$, i.e. $\emme^{'*} = \la \Delta^{'*}, <^{'*}, I^{'*} \ra$,
where $\Delta^{'*} = \Delta \cup \Delta^*$.
As far as individuals named in the ABox,
$I^{'*} = I$, whereas for the concepts and roles, $I^{'*} = I$ on $\Delta$ and $I^{'*} = I'$ on $\Delta'$.
Also, $k_{\emme^{'*}} = k_{\emme}$ for the elements in $\Delta$,
and $k_{\emme^{'*}} = k_{\emme'}$ for the elements in $\Delta'$. $<^{'*}$ is straightforwardly defined from $k_{\emme^{'*}}$ as described just before Definition \ref{definition_rank_formula}.

The model $\emme^{'*}$ is still a model of KB. For the set $K_F$ in the previous definition this is obviously true. For $K_D$, for each $\tip(C ) \sqsubseteq D$ in $K_D$,
if $x \in Min_{<'*}( C)$ in $\emme'*$, also $x \in Min_{<}( C)$ in $\emme$
or $x \in Min_{<'}( C)$ in $\emme'$. In both cases $x$ is an instance of $D$ (since both $\emme$ and $\emme'$ satisfy $K_D$), therefore $x \in  D^{I^{'*}}$, and $\emme^{'*}$ satisfies $K_D$.
 
By repeating the same construction for all  maximal subsets $\{C_1, C_2, \dots, C_n\}$ of $\lingconc$, we obtain a canonical model of KB, call it $\emme^*$. We do not know whether the model is minimal. However by applying the construction used in the  proof of Proposition \ref{modello_minimo_DL}, we obtain  ${\emme^{*}}_{min}$ that is a minimal model of KB with the same domain and interpretation function than $\emme^*$. ${\emme^{*}}_{min}$ is therefore a canonical model of KB, and furthermore it is minimal. Therefore KB has a minimal canonical model.
\vspace{-0.2cm}
$\hfill \bbox$
\end{proof} 

\noindent To prove the correspondence between minimal canonical models and the rational closure of a TBox,
we need to introduce some propositions.
The next one concerns all $\shiqrt$ models. Given a $\shiqrt$ model $\emme=$$\langle \Delta, <, I \rangle$,
we define a sequence $\emme$$_0$, $\emme$$_1, \emme$$_2, \ldots$ of models as follows:
We let  $\emme$$_0=\emme$ and,
for all $i$, we let $\emme$$_i =\langle \Delta,
<_i, I \rangle$ be the $\shiqrt$ model obtained from $\emme$ by
assigning a rank $0$ to all the domain elements $x$ with $k_{\emme}(x) < i$, i.e.,
$k_{{\emme}_{i}}(x)= k_{\emme}(x)-i$ if  $k_{\emme}(x) > i $, and $k_{{\emme}_i}(x)= 0$ otherwise.
We can prove the following:

\begin{proposition}\label{prop:B_DL}
Let KB$=\langle TBox, ABox \rangle$ and
let $\emme=$$\langle \Delta, <, I \rangle$ be any $\shiqrt$ model of TBox.
 For any concept
$C$, if rank($C$) $\geq i$, then 1) $k_{\emme}($C$) \geq i$, and
2) if $\tip(C) \sqsubseteq D$ is
entailed by $E_i$, then $\emme$$_i$ satisfies  $\tip(C) \sqsubseteq D$.
\end{proposition}

\begin{proof} By induction on $i$.
For $i=0$, 1) holds (since it always holds that $k_{\emme}($C$) \geq 0$).
2) holds trivially as $\emme$$_0=\emme$.

For $i>0$, 1) holds: if rank($C$)$\geq i$, then, by Definition \ref{Def:Rank of a formula_DL},
for all $j < i$, we have that $E_{j} \models \tip(\top) \sqsubseteq \neg C$. By inductive hypothesis on 2), for all $j<i$, $\emme$$_j \models \tip(\top) \sqsubseteq \neg C$.
Hence, for all $x$ with $k_{\emme}(x) < i$, $x \not \in C^I$, and $k_{\emme} (C) \geq i$.

To prove 2), we reason as follows. Since $E_i \subseteq E_0$, $\emme$ $\models E_i$. Furthermore by
definition of rank, for all $\tip(C) \sqsubseteq D \in E_i$, rank($C$) $\geq i$, hence by 1) just proved $k_{\emme} (C) \geq i$. Hence, in $\emme$, $Min_{<}(C^I)\geq i$, and also $\emme$$_i \models \tip(C) \sqsubseteq D$. Therefore  $\emme$$_i \models E_i$.
 \vspace{-0.3cm}
$\hfill \Box$ \end{proof}

\noindent  Let us now focus our attention on minimal canonical models by proving the  correspondence between rank of a formula (as in Definition \ref{Def:Rank of a formula_DL})
and rank of a formula in a model (as in Definition \ref{definition_rank_formula}).
The following proposition 
is proved by induction on the rank $i$:

\begin{proposition}\label{proposition_rank}
Given KB and $\lingconc$, for all $C \in \lingconc$, if  $\rf(C)=i$, then:
1. there is a $\{C_1 \dots C_n\} \subseteq \lingconc$ maximal and consistent with KB such that $C \in \{C_1 \dots C_n\}$ and $\rf(C_1 \sqcap \dots \sqcap  C_n) = i$;
2. for any $\emme$ minimal canonical model of KB, $k_{\emme}(C) =i$.\end{proposition}

\begin{proof}
By induction on $i$. Let us first consider the base case in which $i =0$. We have that KB $\not\modelsshiqrt \tip(\top) \sqsubseteq \neg C$. Then there is a minimal  model $\emme$$_1$ of KB with a domain element $x$ such that $k_{{\emme}_{1}}$$(x)=0$ and $x$ satisfies $C$. For 1): consider the maximal consistent set of concepts in $\lingconc$ of which $x$ is an instance in $\emme$$_1$.  This is a maximal consistent $\{C_1 \dots C_n\} \subseteq \lingconc$ containing $C$. Furthermore, $\rf(C_1 \sqcap \dots \sqcap  C_n) = 0$ since
clearly KB $\not\modelsshiqrt \tip(\top) \sqsubseteq \neg (C_1 \sqcap \dots \sqcap  C_n)$. For 2): by definition of canonical model, in any canonical model $\emme$ of KB, $\{C_1 \dots C_n\}$ is satisfiable by an element $x$. Furthermore, in any minimal canonical $\emme$, $k_{\emme}(x)=0$, since otherwise we could build $\emme'$ identical to $\emme$ except from the fact that  $k_{\emme'}(x)=0$. It can be easily proven that $\emme'$ would still be a model of KB (indeed $\{C_1 \dots C_n\}$ was already satisfiable in $\emme$$_1$ by an element with rank 0) and  $\emme' <_{\mathit{FIMS}} \emme$, against the minimality of $\emme$. Therefore, in any minimal canonical model $\emme$ of KB, it holds $k_{\emme}(C)=0$.

For the inductive step, consider the case in which $i>0$. We have that
 $E_i \not\modelsshiqrt \tip(\top) \sqsubseteq \neg C$, then there  must be a model $\emme_1=\la \Delta_1, <_1, I_1\ra$ of $E_i$, and a domain element $x$ such that $k_{{\emme}_1}(x)=0$ and $x$ satisfies $C$. Consider the maximal consistent set of concepts $\{C_1, \dots Cn\} \subseteq \lingconc$ of which $x$ is an instance in $\emme$$_1$.  $C \in \{C_1, \dots Cn\}$. Furthermore, $\rf(C_1 \sqcap \dots \sqcap  C_n) = i$. Indeed $E_{i-1} \modelsshiqrt \tip(\top) \sqsubseteq \neg (C_1 \sqcap \dots \sqcap  C_n)$ (since $E_{i-1} \modelsshiqrt \tip(\top) \sqsubseteq \neg C$ and $C \in \{C_1, \dots Cn\}$), whereas clearly by the existence of $x$,  $E_{i} \not\modelsshiqrt \tip(\top) \sqsubseteq \neg (C_1 \sqcap \dots \sqcap  C_n)$.
In order to prove 1) we are left to prove that the set $\{C_1, \dots Cn\}$ (that we will call $\Gamma$ in the following) is consistent with KB. 

To prove this, take any minimal canonical model  $\emme=\la \Delta, <, I\ra$ of KB. By inductive hypothesis we know that for all concepts $C'$ such that $\rf(C') < i$, there is a maximal consistent set of concepts $\{C'_1, \dots C'_n \}$ with  $C' \in \{C'_1, \dots C'_n \}$ and $\rf(C'_1 \sqcap \dots \sqcap C'_n) =j < i$. Furthermore, we know that  $k_{\emme}(C')  = j < i$. For a contradiction, if $\emme$ did not contain any element satisfying $\Gamma$ we could expand 
it 
by adding to $\emme$ a portion of the model $\emme_1$ including $x \in \Delta_1$.
More precisely, we add to $\emme$ a new set of domain elements $\Delta_x \subseteq \Delta_1$, 
containing the domain element  $x$ of $\emme_1$  and all the domain elements of $\Delta_1$ which are reachable from $x$ in $\emme_1$ through a sequence of relations $R_i^{I_1}$s  or $(R^{-}_i)^{I_1}$s. 
Let $\emme'$  be the resulting model. We define $I'$ on the elements of $\Delta$ as in $\emme$, while we define $I'$ on the element of $\Delta_x$  as in $I_1$.
Finally, we let, for all $w \in \Delta$, $k_{\emme'}(w) =k_{\emme}(w)$
and,  for all $y \in \Delta_x$, $k_{\emme'}(y) =i+k_{\emme_1}(y)$.
In particular, $k_{\emme'}(x) =i$.
The resulting model  $\emme'$ would still be a model of KB.  Indeed, the ABox would still be satisfied by the resulting model (being the $\emme$ part unchanged). For the TBox: all domain elements  already in $\emme$ still satisfy all the inclusions. For all $y \in \Delta_x$ (including $x$): for all inclusions in $E_i$, $y$ satisfies them (since it did it in $\emme_1$); for all typicality inclusions $\tip(D) \sqsubseteq G \in$ KB $- E_i$, $\rf(D) < i$, hence by inductive hypothesis $k_{\emme}(D) < i$, hence $k_{\emme'}(D) < i$, and $y$ is not a typical instance of $D$ and trivially satisfies the inclusion. 
It is easy to see that $\emme'$ also satisfies role inclusions $R \sqsubseteq S$ and that, for each transitive roles $R$, $R^{I'}$ is transitive.

We have then built a model of KB satisfying $\Gamma$. Therefore $\Gamma$  is consistent with KB, and therefore by definition of canonical model, $\Gamma$ must be satisfiable in $\emme$. Up to now we have proven that $\Gamma$ is maximal and consistent with KB, it contains $C$ and has rank $i$, therefore point 1) holds. 

In order to prove point 2) we need to prove that any minimal canonical model $\emme$ of KB not only satisfies $\Gamma$ but it satisfies it with rank $i$, i.e. $k_{\emme}(C_1 \sqcap \dots \sqcap  C_n)= i$, which entails $k_{\emme}(C)= i$ (since $C \in \{C_1, \dots C_n \}$). By Proposition \ref{prop:B_DL} we know that $k_{\emme}(C_1 \sqcap \dots \sqcap  C_n) \geq i$. We need to show that also $k_{\emme}(C_1 \sqcap \dots \sqcap  C_n) \leq i$. We reason as above: for a contradiction suppose $k_{\emme}(C_1 \sqcap \dots \sqcap  C_n) > i$, i.e., for all the minimal domain elements $y$ instances of $C_1 \sqcap \dots \sqcap  C_n$,  $k_{\emme}(y) > i$. We show that this contradicts the minimality of $\emme$. Indeed consider $\emme'$ obtained from $\emme$ by letting  $k_{\emme'}(y) = i$, for some minimal domain element $y$ instance of $C_1 \sqcap \dots \sqcap  C_n$, and leaving all the rest unchanged. $\emme'$ would still be a model of KB: the only thing that changes with respect to $\emme$ is that $y$ might have become in $\emme'$ a minimal instance of a concept of which it was only a non-typical instance in $\emme$. This might compromise the satisfaction in $\emme$ of a typical inclusion as $\tip(E) \sqsubseteq G$. However: if $\rf(E) < i$, we know by inductive hypothesis that $k_{\emme}(E) < i$ hence also $k_{\emme'}(E) < i$ and $y$ is not a minimal instance of $E$ in $\emme'$.  
If $\rf(E) \geq i$, then $\tip(E) \sqsubseteq G \in E_i$. As $y\in C_1 \sqcap \dots \sqcap  C_n$ (where $\{C_1, \dots C_n \}$ is maximal consistent with KB), we have that:
$y \in F^I$ iff $x \in F^{I_1}$, for all concepts $F$.
If $y \in E^I$, then $E \in \{C_1, \dots C_n \}$. Hence, in $\emme_1$, $x \in E^{I_1}$. But $\emme_1$ is a model of $E_i$, and satisfies all the inclusions in $E_i$.
Therefore $x \in G^{I_1}$ and, thus, $y \in G^I$.

It follows that $\emme'$ would be a model of KB, and $\emme'$$ <_{\mathit{FIMS}} \emme$, against the minimality of $\emme$. We are therefore forced to conclude that $k_{\emme}(C_1 \sqcap \dots \sqcap  C_n)= i$, and hence also $k_{\emme}(C)= i$, and 2) holds. 
\vspace{-0.3cm}
 
$\hfill \Box$ \end{proof}

\noindent The following theorem follows  from the propositions above:

\begin{theorem}\label{Theorem_RC_TBox}
Let KB=(TBox,ABox) be a knowledge base and $C \sqsubseteq D$ a query.
We have that $C \sqsubseteq D \in$ $\overline{\mathit{TBox}}$ if and only if $C \sqsubseteq D$ holds in all minimal  canonical models of KB with respect to $\lingconc$.
\end{theorem}

\begin{proof}
{\em (Only if part)}
Assume that $C \sqsubseteq D$ holds in all minimal  canonical models
of KB with respect to $\lingconc$, and let $\emme=$$\langle \Delta, <, I \rangle$
be a minimal canonical model of KB satisfying $C \sqsubseteq D$.
Observe that $C$ and $D$ (and their complements) belong to $\lingconc$.
We consider two cases: (1) the left end side of the inclusion $C$ does not contain the typicality operator, and (2) the left end side of the inclusion is $\tip(C)$.

In case (1), if the minimal canonical model $\emme$ of KB satisfies $C \sqsubseteq D$.
Then, $C^I \subseteq D^I$.
For a contradiction, let us assume that $C \sqsubseteq D \not \in$ $\overline{\mathit{TBox}}$.
Then, by definition of $\overline{\mathit{TBox}}$, it must be: KB $\not \models_{\shiq} C \sqsubseteq D$.
Hence, KB $\not \models_{\shiq} C \sqcap \neg D \sqsubseteq \bot$,
and the set of concepts $\{C, \neg D\}$ is consistent with KB.
As $\emme$ is a canonical model of KB, there must be a element $x \in \Delta$ such that
$x \in (C \sqcap \neg D)^I$.
This contradicts the fact that $C^I \subseteq D^I$.

In case (2), assume $\emme$ satisfies $\tip(C) \sqsubseteq D$.
Then, $\tip(C)^I \subseteq D^I$, i.e., for each $x \in Min_{<}(C^I)$, $x \in D^I$.
 If $Min_{<}(C^I) = \emptyset$, then
there is no $x \in C^I$ (by the smoothness condition),
hence $C$ has no rank
in $\emme$ and, by Proposition \ref{proposition_rank},
$C$ has no rank ($\rf(C)=\infty$). In this case, by Definition \ref{def:rational closureDL}, $\tip(C) \sqsubseteq D \in
\overline{\mathit{TBox}}$.
Otherwise, 
let us assume that $k_{\emme}(C) = i$.
As
$k_{\emme}(C \sqcap D) < k_{\emme}(C \sqcap \neg D)$, then
$k_{\emme}(C \sqcap \neg D)> i$. By Proposition \ref{proposition_rank},
$rank(C)=i$ and $rank(C \sqcap \neg D)> i$. Hence, by Definition
\ref{def:rational closureDL}, $\tip(C) \sqsubseteq D \in \overline{\mathit{TBox}}$.

{\em (If part)}
If $C \sqsubseteq D \in \overline{\mathit{TBox}}$, then, by definition of $ \overline{\mathit{TBox}}$,
KB $\models_{\shiq} C \sqsubseteq D$.
Therefore, each minimal canonical model $\emme$ of KB satisfies $C \sqsubseteq D$.

If $\tip(C) \sqsubseteq D \in \overline{\mathit{TBox}}$,
then by  Definition \ref{def:rational closureDL},
either (a) $rank(C) < rank(C \sqcap \neg D)$, or (b) $C$ has no
rank.
Let $\emme$ be any minimal canonical model of KB.
In the case (a), by Proposition \ref{proposition_rank}, $k_{\emme}(C) < k_{\emme}(C \sqcap \neg D)$, which entails $k_{\emme}(C \sqcap D) < k_{\emme}(C \sqcap \neg D)$. Hence $\emme$ satisfies $\tip(C) \sqsubseteq D$.
In case (b), by Proposition \ref{proposition_rank}, $C$ has no rank in $\emme$, hence $\emme$ satisfies $\tip(C) \sqsubseteq D$.
 \vspace{-0.3cm}
$\hfill \Box$\end{proof}


\hide{

\begin{theorem}[Complexity of rational closure over TBox]
Given a knowledge base KB$=(TBox,ABox)$, the problem of deciding whether
$\tip(C) \sqsubseteq D \in \overline{\mathit{TBox}}$ is in \textsc{ExpTime}.
\end{theorem}

\begin{proof}
Checking if $\tip(C) \sqsubseteq D \in \overline{\mathit{TBox}}$ can be done by
computing the finite sequence  $E_0, E_1, \ldots, E_n$ of non increasing subsets of TBox inclusions
in the construction of the rational closure.
Note that the number $n$ of the $E_i$ is  $O(|\mbox{KB}|)$,
where $|\mbox{K}|$ is the size of the KB.
Computing each $E_i ={\cal E}(E_{i-1})$, requires to check, for all concepts $C'$
occurring on the left hand side of an inclusion in the TBox, whether $E_{i-1} \modelsshiqrt \tip(\top) \sqsubseteq \neg C'$,
which requires an exponential time in the size of $E_{i-1}$ (and hence in the size of KB).
If not already checked, the exceptionality of $C$ and of $C\sqcap \neg D$ have to be checked for each $E_i$,
to determine the ranks of $C$ and of $C\sqcap \neg D$ (which also requires an exponential time in the size of KB).
Hence, verifying if $\tip(C) \sqsubseteq D \in \overline{\mathit{TBox}}$ is in \textsc{ExpTime}.
\vspace{-0.2cm} $\hfill \bbox$ \end{proof}

}

\section{Rational Closure over the ABox} \label{Section:Abox}
\vspace{-0.3cm}


The definition of rational closure in Section \ref{Rat_Closure_DL} takes only into account the TBox. We address the issue of ABox reasoning first by  the semantical side: \hide{in order to treat individuals explicitly mentioned in the ABox in a uniform way with respect to the other domain elements:} as for any  domain element, we would like to attribute to each individual constant named in the ABox the lowest possible rank. Therefore we further refine Definition \ref{Preference between models wrt TBox} of minimal canonical models with respect to TBox by taking into account the interpretation of individual constants of the ABox. 

\begin{definition}[Minimal canonical model w.r.t. ABox]\label{model-minimally-satisfying-k}
Given KB=(TBox,ABox), let $\emme = $$\langle \Delta, <, I \rangle$ and $\emme' =
\langle \Delta', <', I' \rangle$ be two canonical models of KB which are minimal w.r.t. Definition \ref{Preference between models wrt TBox}. We say that $\emme$ is preferred to $\emme'$ w.r.t. ABox ($\emme <_{\mathit{ABox}} \emme'$) if, for all individual constants $a$ occurring in ABox, $k_{\emme}(a^I) \leq k_{\emme'}(a^{I'})$ and there is at least one individual constant $b$ occurring in ABox such that  $k_{\emme}(b^I) < k_{\emme'}(b^{I'})$.
\end{definition}

\noindent As a consequence of Proposition  \ref{modello_minimo_canonico_DL}
 we can prove that:
\begin{theorem}\label{modello_minimo_DL_ABox}
For any KB$=(TBox, ABox)$ there exists a minimal canonical model of KB with respect to ABox. 
\end{theorem}

%

\noindent In order to see the strength of the above semantics, consider our  example about marriages and VIPs.
\vspace{-0.15cm}
\begin{example}\label{example-ABox-semantic1}
Suppose we have a  KB=(TBox,ABox) where:  TBox=$\{\tip (\mathit{Person}) \sqsubseteq \ \leq 1 \ \mathit{HasMarried}.\mathit{Person},$  
$\tip (\mathit{VIP}) \sqsubseteq \ \geq 2 \ \mathit{HasMarried}.\mathit{Person},$ 
$\mathit{VIP} \sqsubseteq \mathit{Person} \},$
and  ABox = $\{\mathit{VIP}(\mathit{demi}), \mathit{Person}(\mathit{marco})\}$.
 Knowing  that Marco is a person and Demi is a VIP, we would like to be able to assume, in the absence of other information, that Marco is a typical person, whereas Demi is a typical VIP, and therefore Marco has at most one spouse, whereas Demi has at least two. Consider any minimal canonical model $\emme$ of KB. Being canonical, $\emme$ will contain, among other elements, the following:

\vspace{0.1cm}
\begin{footnotesize}
\noindent $x \in (\mathit{Person})^I$, $x \in (\leq 1 \ \mathit{HasMarried}.\mathit{Person})^I$, $x \in (\neg \mathit{VIP})^I$, $k_{\emme}(x)=0$;

\noindent $y \in (\mathit{Person})^I$, $y \in (\geq 2 \ \mathit{HasMarried}.\mathit{Person})^I$, $y \in (\neg \mathit{VIP})^I$, $k_{\emme}(y)=1$;

\noindent $z \in (\mathit{VIP})^I$, $z \in (\mathit{Person})^I$, $z \in (\geq 2 \ \mathit{HasMarried}.\mathit{Person})^I$ , $k_{\emme}(z)=1$;

\noindent $w \in (\mathit{VIP})^I$, $w \in (\mathit{Person})^I$, $w \in (\leq 1 \ \mathit{HasMarried}.\mathit{Person})^I$ , $k_{\emme}(w)=2$.
\vspace{0.1cm}
\end{footnotesize}

\noindent so that $x$ is a typical person and $z$ is a typical VIP. Notice that in the definition of minimal canonical model there is no constraint on the interpretation of   constants $\mathit{marco}$ and $\mathit{demi}$. As far as Definition \ref{Preference between models wrt TBox} is concerned, for instance, $\mathit{marco}$ can be mapped onto $x$ ($(\mathit{marco})^I=x$) or onto $y$ ($(\mathit{marco})^I=y$): the minimality of $\emme$ w.r.t. Definition \ref{Preference between models wrt TBox} is not affected by this choice. However in the first case it would hold that Marco is a typical person, in the second Marco is not a typical person. 
According to Definition \ref{model-minimally-satisfying-k}, we prefer the first case, and there is a unique minimal canonical model w.r.t. ABox in which $(\mathit{marco})^I=x$ and $(\mathit{demi})^I=z$.
\end{example}

\noindent We next provide an algorithmic construction for the rational closure of ABox. 
The idea is that of considering all the possible minimal consistent
assignments of ranks to the individuals explicitly named in the ABox. Each
assignment adds some properties to named individuals which can be used to infer
new conclusions. We adopt a skeptical view by considering only those
conclusions which hold for all assignments. The equivalence with the semantics
shows that the minimal entailment captures a skeptical approach when reasoning
about the ABox.
More formally, in order to calculate the rational closure of ABox, written $\overline{\mathit{ABox}}$, for all individual constants of the ABox we find out which is the lowest possible rank they can have in minimal canonical models with respect to Definition \ref{Preference between models wrt TBox}:  the idea is that an individual constant $a_i$ can have a given rank $k_j(a_i)$ just in case it is compatible with all the inclusions $\tip(A)  \sqsubseteq D$ of the TBox whose antecedent $A$'s rank is  $\geq  k_j(a_i)$ 
(the inclusions whose antecedent $A$'s rank is $< k_j(a_i)$ do not matter since, in the canonical model, there will be an instance of $A$ with rank $<  k_j(a_i)$ and therefore $a_i$ will not be a typical instance of $A$). 
The algorithm below computes all minimal rank assignments $k_j$s to all individual constants:
 $\mu^j_{i}$ contains all the concepts that $a_i$ would need to satisfy in case it had the rank attributed by $k_j$ ($k_j(a_i)$). The algorithm verifies whether $\mu^j$ is compatible with ($\overline{\mathit{TBox}}$, ABox) and whether it is minimal. Notice that, in this phase, all constants are considered simultaneously (indeed, the possible ranks of different individual constants depend on each other).
For this reason  $\mu^j $ takes into account the ranks attributed to all individual constants, being the union of all $\mu^j_{i}$ for all $a_i$, and the consistency of this union with ($\overline{\mathit{TBox}}$, ABox) is verified. 

\begin{definition}[$\overline{\mathit{ABox}}$: rational closure of ABox]\label{Rational Closure ABox}
     Let $a_1, \dots, a_m$ be the individuals explicitly named in the ABox. Let
$k_1, k_2, \dots, k_h$ be all the possible rank assignments (ranging from $1$
to $n$) to the individuals occurring in ABox.

\noindent -- Given a rank assignment $k_j$  we define:
\begin{itemize}
\vspace{-0.3cm}
\item for each $a_i$: $\mu^j_{i} = \{ (\neg C \sqcup D)(a_i)$
s.t. $C, D \in \lingconc$, $\tip(C) \sqsubseteq D$ in
$\overline{\mathit{TBox}}$, and  $k_j(a_i) \leq rank(C)  \} \cup \{ ( \neg C
\sqcup D)(a_i) $ s.t. $C \sqsubseteq D$ in TBox $\}$;
\item let $\mu^j = \mu^j_{1} \cup \dots \cup \mu^j_{m} $ for all $\mu^j_{1}
\dots \mu^j_{m} $ just calculated for all $a_1, \dots, a_m$ in  ABox
\end{itemize}

\vspace{-0.15cm}
\noindent -- $k_j$ is \emph{minimal and consistent} with ($\overline{\mathit{TBox}}$, ABox), i.e.:
(i) TBox $\cup$ ABox $\cup \mu^j$ is consistent in $\shiqrt$;
(ii) there is no $k_i$ consistent with ($\overline{\mathit{TBox}}$, ABox) s.t. for all $a_i$, $k_i(a_i) \leq k_j(a_i)$ and for some $b$, $k_i(b) < k_j(b)$.

\noindent -- The rational closure of ABox (
$\overline{\mathit{ABox}}$) is the set of all assertions derivable in $\shiqrt$ from
TBox $\cup$ ABox $\cup \mu^j$ for all minimal consistent rank assignments $k_j$, i.e:
\begin{center}
\vspace{-0.2cm}
\begin{footnotesize}
$\overline{\mathit{ABox}} = \bigcap_{k_j \mbox{{\tiny minimal consistent}}}\{C(a): \;$ TBox $\cup$ ABox $\cup \ \mu^j \
\models_{\shiqrt} C(a) \}$
\end{footnotesize}
\vspace{-0.2cm}
\end{center}
\end{definition}

\noindent 
The  example below is the syntactic counterpart of the semantic Example \ref{example-ABox-semantic1} above.

\vspace{-0.2cm}

\begin{example}\label{example-ABox-algorithm1}
Consider the KB in Example \ref{example-ABox-semantic1}.
%
Computing the ranking of concepts we get that $rank(\mathit{Person}) = 0$, $rank(\mathit{VIP}) = 1$, $rank(\mathit{Person} \ \sqcap \ \geq 2 \ \mathit{HasMarried}.\mathit{Person}) = 1$,
$rank(\mathit{VIP} \ \sqcap \ \leq 1 \ \mathit{HasMarried}.\mathit{Person}) = 2$. It is easy to see that a rank assignment $k_0$ with $k_0(\mathit{demi}) = 0$ is inconsistent with  KB as $\mu^0$ would contain $(\neg \mathit{VIP} \sqcup \mathit{Person})(\mathit{demi})$,  $(\neg \mathit{Person} \ \sqcup \ \leq 1 \ \mathit{HasMarried}.\mathit{Person})(\mathit{demi})$, $(\neg \mathit{VIP} \sqcup \ \geq \ 2 \ \mathit{HasMarried}.$ $\mathit{Person})(\mathit{demi})$ and $\mathit{VIP}(\mathit{demi})$. Thus we are left with only two ranks $k_1$ and $k_2$ with respectively $k_1(\mathit{demi}) = 1, k_1(\mathit{marco}) = 0$  and $k_2(\mathit{demi}) =  k_2(\mathit{marco}) = 1$.

The set $\mu^1$ contains, among the others, $(\neg \mathit{VIP} \ \sqcup \ \geq 2 \ \mathit{HasMarried}.\mathit{Person})(\mathit{demi})$ ,  $(\neg \mathit{Person} \ \sqcup \ \leq 1 \ \mathit{HasMarried}.\mathit{Person})(\mathit{marco})$.
It is tedious but easy to check that KB $\cup \mu^1$ is consistent and that $k_1$ is the only minimal  consistent assignment (being $k_1$ preferred to $k_2$),   thus both $(\geq 2 \ \mathit{HasMarried}.\mathit{Person})(\mathit{demi})$ and $(\leq 1 \ \mathit{HasMarried}.\mathit{Person})$ $(\mathit{marco})$ belong to $\overline{\mathit{ABox}}$.
\end{example}

\noindent We are now ready to show the soundness and completeness of the algorithm with respect to the semantic definition of rational closure of ABox.


\begin{theorem}[Soundness of $\overline{\mathit{ABox}}$]\label{Theorem_Soundness_ABox}
Given KB=(TBox, ABox), for each individual constant $a$ in ABox, we have that
if $C(a) \in$ $\overline{\mathit{ABox}}$ then $C(a)$ holds in all minimal
canonical models with respect to ABox of KB.
\end{theorem}
\begin{proof}[Sketch]
Let $C(a) \in $
$\overline{\mathit{ABox}}$, and suppose for a contradiction that there is a
minimal canonical model $\emme$ with respect to ABox of KB s.t. $C(a)$ does not hold in
$\emme$. Consider now the rank assignment $k_j$ corresponding to $\emme$ (such that $k_j(a_i)=k_{\emme}(a_i)$). 
By hypothesis $\emme$ $\models $ TBox $\cup$ ABox. Furthermore it can be easily shown that $\emme$ $\models \mu^j$.

Since by hypothesis $\emme$$ \not\models C(a)$, it follows that TBox $\cup$ ABox $\cup  \ \mu^j
\not\models_{\shiqrt} C(a)$, and by definition of
$\overline{\mathit{ABox}}$, $C(a) \not\in \overline{\mathit{ABox}}$, against the hypothesis.
$\hfill \Box$ \end{proof}

\begin{theorem}[Completeness of $\overline{\mathit{ABox}}$]\label{Theorem_Completeness_ABox}
Given KB=(TBox, ABox), for all  individual constant $a$ in ABox, we have that
if $C(a)$ holds in all minimal canonical models with respect to ABox of KB, then
$C(a) \in$ $\overline{\mathit{ABox}}$.
\end{theorem}
\begin{proof}[Sketch]
We show the contrapositive.
Suppose $C(a) \not\in $ $\overline{\mathit{ABox}}$, i.e. there is a minimal
$k_j$ consistent with ($\overline{\mathit{TBox}}$, ABox) s.t. TBox $\cup$ ABox $\cup \mu^j
\not\models_{\shiqrt} C(a)$. This means that there is an  $\emme' = \langle \Delta', <, I' \rangle$ such that for all $a_i \in ABox$, $k_{\emme'}(a_i) = k_j(a_i) $, 
$\emme' $$\models $ TBox $\cup$ ABox $\cup \mu^j$ and $\emme' \not\models C(a)$.
From $\emme'$ we build a minimal canonical model with respect to ABox $\emme = $$\langle \Delta, < , I \rangle$ of KB, 
such that $C(a_i)$ does not hold in $\emme$.

Since we do not know whether $\emme'=$$\langle \Delta', <', I' \rangle$ is minimal or canonical, we cannot use it directly; rather, we only use it as a support to the construction of $\emme$. 
As the TBox is satisfiable, by Theorem 46, we know that there exists a minimal canonical model $\emme''=$$\langle \Delta'', <'', I'' \rangle$ of the TBox. We extend such a model with domain elements from $\Delta'$ including those elements
interpreting the individuals $a_1, \ldots, a_m$ explicitly named in the ABox.
Let $\Delta =  \Delta_1 \cup \Delta''$ where 
$\Delta_1= \{ (a_i)^{I'}:$ $a_i$ in ABox $\} \cup$ 
$ \{x \in \Delta':$  $x$ is reachable from some $(a_i)^{I'}$ in $\emme'$ by a sequence of $R^{I'}$ or $(R^-)^{I'}\}$.

We define the rank $k_{\emme}$ of each domain element in $\Delta$ as follows.
For the elements $y \in \Delta''$, $k_{\emme}(y) = k_{\emme''}(y)$.
For the elements $x \in \Delta_1$, if $x=(a_i)^{I'}$, then  $k_{\emme}(x) = k_{\emme'}(x)$;
 if $x \neq (a_i)^{I'}$, then  $k_{\emme}(x) = k_{\emme''}(X)$, for some $X \in \Delta''$ such that 
for all concepts $C' \in {\cal S}$, we have $x \in (C')^{I'}$ if and only if $X \in (C')^{I''}$.
%
We then define $I$ as follows.
%
First, for all $a_i$ in ABox we let $a_i^I =(a_i)^{I'}$. 
We define the interpretation of each concept as in $\Delta'$ on the elements of $\Delta_1$ 
and  as in $\Delta''$ on the elements of $\Delta''$.
Last, we define the interpretation of each role $R$ as in $\emme'$ on the pairs of elements of $\Delta_1$ 
and  as in $\emme''$ on  the pairs of  elements of $\Delta''$.
$I$ is extended to quantified concepts in the
usual way.

It can be proven that $\emme$ satisfies ABox (by definition of $I$ and since $\emme'$ satisfies it).
Furthermore it can be proven that $\emme$ satisfies TBox (the full proof is omitted due to space limitations).
$C(a)$ does not hold in $\emme$, since it does not hold in $\emme'$.
Last, $\emme$ is 
canonical by construction. It is minimal
with respect to Definition \ref{Preference between models wrt TBox}: for all $X \in \Delta_2$ $k_{\emme}(X)$ is the lowest possible rank it can have in any model (by Proposition \ref{proposition_rank}); for all $a_i \in \Delta_1$, this follows by minimality of $k_j$. From minimality of $k_j$ it also follows that $\emme$ is a minimal canonical model with respect to ABox. Since in $\emme$   $C(a)$ does not
hold, the theorem follows by contraposition.
$\hfill \Box$  \end{proof}

\begin{theorem}[Complexity of rational closure over the ABox]
Given a knowledge base KB=(TBox,ABox) in $\shiqrt$, an individual constant $a$ and a concept $C$, the problem of deciding whether $C(a) \in \overline{\mathit{ABox}}$ is \textsc{ExpTime}-complete.
\end{theorem}

\noindent We omit the proof, which is similar to the one for rational closure over  ABox in $\alc$ (Theorem 5 \cite{dl2013}).


\section{Extending the correspondence to more expressive logics} \label{Section:Shoiq}

A natural question is whether  the correspondence between the rational closure and the minimal canonical model semantics of the previous section can be extended to stronger DLs. We give  a negative answer for the logic ${\cal SHOIQ}$.
This depends on the fact that, due to the interaction of nominals with number restriction, a consistent ${\cal SHOIQ}$
knowledge base may have no canonical models (whence no minimal canonical ones). Let us consider for instance the following example:

\begin{example}
Consider the KB, where
TBox$= \{ \{o\} \sqsubseteq \leq 1 R^-.\top, \;  \neg \{o\} \sqsubseteq \ \geq 1 R.\{o\}  \}$, and
ABox$=\{\neg A(o), \neg B(o)\}$.

KB is consistent and, for instance, the model $\emme_1= \la \Delta, <, I\ ra$, where $\Delta =\{x,y\}$,
$<$ is the empty relation, $A^I= B^I= (\neg  \{o\} )^I=  \{x\}$, and $(\{o\} )^I=  \{y\}$, is a model of KB.
In particular, $x \in (A \sqcap B)^I$.

Also, there is a model $\emme_2$  of KB similar  to $\emme$ (with $\Delta_2=\{ x_2, y\}$) in which  $x_2 \in (A \sqcap \neg B)^I$,
another one $\emme_3$  (with $\Delta_3=\{ x_2, y\}$) in which  $x_3 \in (\neg A \sqcap B)^I$, and so on.
Hence, $\{A, B\}$, $\{A,\neg B\}$, $\{\neg  A, B\}$, $\{\neg A,\neg B\}$ are all sets of concepts $\lingconc$
that are consistent with KB.
Nevertheless, there is no canonical model for KB containing $x_1$, $x_2$ and $x_3$ all together.
as the inclusions in the TBox prevent models from containing more than two domain elements.


\end{example}

The above example shows that the notion of canonical model as defined in this paper is too strong to capture the notion of rational closure for logics which are as expressive as ${\cal SHOIQ}$.
Beacause of this negative result, we can regard the correspondence result for ${\cal SHIQ}$ only as a first step in the definition of a semantic characterization of rational closure for expressive description logics. A suitable refinement of the semantics is needed, and we leave its definition for future work.

\vspace{-0.2cm}

\section{Related Works}
\vspace{-0.3cm}


There are a number of works which are closely related to our proposal. 

In \cite{FI09,AIJ} nonmonotonic extensions of DLs based on the \tip \ operator have been proposed. In these extensions, focused on the basic DL $\alc$, the semantics of \tip \ is based on preferential logic \Pe. 
Moreover and more importantly, the notion of minimal model adopted here is completely independent from the language and  is  determined only by the relational structure of models. 

\cite{casinistraccia2010} develop a notion of rational closure for DLs. They propose a construction to compute the rational closure of an $\alc$  knowledge base, which is not directly based on Lehmann and Magidor definition of rational closure, but is similar to the construction of rational closure proposed by  Freund \cite{freund98} at a propositional level.  
\normalcolor
\hide{
\cite{casinistraccia2010} keeps the ABox into account, and defines closure operations over individuals. 
It introduces a consequence relation $\Vdash$ among a KB and assertions, under the requirement that the TBox is unfoldable and
 the ABox is closed under completion rules, such as, for instance, that if 
$a: \esiste R.C \in$ ABox, then both $aRb$ and $b: C$ (for some individual constant $b$) must belong to the ABox too. 
Under such restrictions, a procedure is defined to  compute the rational closure of the ABox, assuming that the individuals explicitly named are linearly ordered, and different orders determine different sets of consequences. It is shown that, for each order $s$, the consequence relation $\Vdash_s$ is rational and can be computed in \textsc{PSpace}.
}
 In a subsequent work, \cite{Casinistraccia2011} introduces an approach based on the combination of rational closure and \emph{Defeasible Inheritance Networks} (INs).
In \cite{CasiniDL2013},  a work on the semantic characterization of
a variant of the notion of rational closure introduced in \cite{casinistraccia2010} has been presented, based on a generalization to $\alc$ of our semantics in \cite{jelia2012}.

An approach related to ours can be found in \cite{BoothParis}.
The basic idea of their semantics is similar to ours, but it is restricted to the propositional case. ÊFurthermore, their construction relies on a specific representation of models and it provides a recipe to build a model of the rational closure, rather than a characterization of its properties. Our semantics, defined in terms of Êstandard Kripke models, can be more easily generalized to richer languages, as we have done here for $\shiq$.

In \cite{sudafricanisemantica} the semantics of the logic of defeasible subsumptions is strengthened by a preferential semantics. 
Intuitively, given a TBox, the authors first introduce a preference ordering $\ll$ on the class of all
subsumption relations $\subsud$ including TBox, then they define the rational closure of TBox as the most preferred relation $\subsud$ with respect to $\ll$, i.e. such that there is no other relation $\subsud'$ such that TBox $\incluso \subsud'$ and $\subsud' \ll \subsud$.
Furthermore, the authors describe an \textsc{ExpTime} algorithm in order to compute the rational closure of a given TBox in $\alc$.  \cite{sudafricanisemantica} does not address the problem of dealing with the ABox. \hide{This algorithm is essentially equivalent to our definition of rational closure for a TBox (Definition \ref{def:rational closureDL}), from which it is straightforward to conclude that the semantics $\mathit{FIMS}$ we have proposed is sound and complete also with respect to the semantics for rational closure proposed in \cite{sudafricanisemantica}, if we restrict our concern to the TBox.}
 In \cite{sudafricaniprotege} a plug-in for the Prot\'eg\'e ontology editor 
implementing the mentioned algorithm for computing the rational closure for a TBox for OWL ontologies is described.

Recent works discuss the combination of open and closed world reasoning in DLs. In particular, formalisms have been defined for combining DLs with logic programming rules (see, for instance, \cite{eiter2004} and \cite{rosatiacm}). A grounded circumscription approach for DLs with local closed world capabilities has been defined in \cite{hitzlerdl}.

\vspace{-0.15cm}

\section{Conclusions}
\vspace{-0.3cm}
In this work we have proposed an extension of the rational closure defined by Lehmann and Magidor to the Description Logic $\shiq$, taking into account both TBox and ABox reasoning. Defeasible inclusions are expressed by means of a typicality operator  $\tip$ which selects the typical instances of a concept. 
One of the contributions is that of extending the semantic characterization of rational closure 
proposed in \cite{jelia2012} for propositional logic, to $\shiq$, which does not enjoy the finite model property.
In particular, we have shown that in all minimal models of a finite KB in $\shiq$ the rank of domain elements is always finite, although the domain might be infinite, and we have exploited this result to establish the correspondence between the minimal model semantics and the rational closure construction for $\shiq$. The (defeasible) inclusions belonging to the rational closure of a $\shiq$ KB correspond to those that are minimally entailed by the KB, when restricting  to canonical models.
We have provided  some complexity results, namely that, for $\shiq$, the problem of deciding whether an inclusion belongs to the rational closure of the TBox  is in \textsc{ExpTime} as well as the problem of deciding whether $C(a)$ belongs to the rational 
closure of the ABox. 
Finally, we have shown that the rational closure of a TBox 
can be computed simply using entailment in $\shiq$.

The rational closure construction in itself can be applied to any description logic. We would like to extend its semantic characterization to stronger logics, such as ${\cal SHOIQ}$, for which the notion of canonical model as defined in this paper is too strong, as we have seen in section \ref{Section:Shoiq}.

\hide{This depends on the fact that, due to the interaction of nominals with number restrictions, a consistent ${\cal SHOIQ}$ 
knowledge base may have no minimal canonical models. 

Let us consider for instance the following example: 

\begin{example} 
Consider the KB, where 
TBox$= \{ \{o\} \sqsubseteq \ \leq 1 R^-.\top, \;  \neg \{o\} \sqsubseteq \ \geq 1 R.\{o\}  \}$, and 
ABox$=\{\neg A(o), \neg B(o)\}$. 
KB is consistent and, for instance, the model $\emme_1= \la \Delta,<, I\ra$, where $\Delta =\{x_1,y\}$, 
$<$ is the empty relation, $A^I= B^I= (\neg  \{o\} )^I=  \{x_1\}$, $(\{o\} )^I=  \{y\}$ and $(x_1,y) \in R^I$, is a model of KB. 
In particular, $x \in (A \sqcap B)^I$. 
Also, there is a model $\emme_2$  of KB similar  to $\emme$ (with $\Delta_2=\{ x_2, y\}$) in which  $x_2 \in (A \sqcap \neg B)^I$, 
another one $\emme_3$  (with $\Delta_3=\{ x_2, y\}$) in which  $x_2 \in (\neg A \sqcap B)^I$, and so on. 
Hence, $\{A, B\}$, $\{A,\neg B\}$, $\{\neg  A, B\}$, $\{\neg A,\neg B\}$ are all sets of concepts 
that are consistent with KB. 
Nevertheless, there is no canonical model for KB containing $x_1$, $x_2$ and $x_3$ all together,
as the inclusions in the TBox prevent models from containing more than two domain elements. 
\end{example} 

The above example shows that the notion of canonical model as defined in this paper is too strong to capture the notion of rational closure for logics which are expressive as ${\cal SHOIQ}$. 
Due to this negative result, we can regard the correspondence result for ${\cal SHIQ}$ only as a first step in the definition of a semantic characterization of rational closure for expressive DLs. A refinement of the semantics is needed to make it weaker, and we leave its definition for future work. \\
}

It is well known that rational closure has some weaknesses that accompany its well-known qualities. 
Among the weaknesses is the fact that one cannot separately reason property by property,
so that, if a subclass of $C$ is exceptional for a given aspect, it is exceptional ``tout court'' and does not inherit any of the
typical properties of $C$.
Among the strengths there is its computational lightness,  which is crucial in Description Logics. Both the qualities and the weaknesses  seems to be inherited by its extension to Description Logics. 
To address the mentioned weakness of rational closure, we may think of attacking the problem from a semantic point of view by  considering a finer semantics where  models are equipped with several preference relations; in such a semantics it might be possible to relativize the notion of typicality, whence to reason about typical properties independently from each other. 

\vspace{0.2cm} 
\noindent {\bf Acknowledgement}. We thank the anonymous referees for their helpful comments. This work has been partially supported by  the Compagnia di San Paolo and by the project ``CONDESC: deduzione automatica per logiche CONDizionali e DESCrittive''.

\bibliographystyle{splncs03}
\bibliography{bibliokr14}

\newpage
\begin{appendix}

\section{APPENDIX: Encoding $\shiqrt$ in $\shiq$}\label{appendiceencoding}

In this section, we provide an encoding of $\shiqrt$ in $\shiq$ and show that
reasoning in $\shiqrt$ has the same complexity as reasoning in $\shiq$.
To this purpose, we first need 
to show that among $\shiqrt$ models, we can restrict our consideration to models where the rank of each element is finite
and less than the number of (sub)concepts occurring in the KB, which is polynomial in the size of the KB.


\begin{proposition} \label{Prop:catene_corte}
Given a  knowledge base KB= (TBox, ABox) in $\shiqrt$, 
there is an $h_{KB}\in  \mathbb{N}$ 
such that, 
for each  model $M$ of the KB in $\shiqrt$ satisfying a concept $C$, there exists a model $M'$ of the KB such that the rank of each element in $M'$ is finite and less then $h_{KB}$, satisfying the concept $C$. Also, $h_{KB}$ is polynomial in the size of the KB.
\end{proposition}
\begin{proof}{\em (Sketch)}
Given a $\shiqrt$ model $\emme =  \langle \Delta, <, I \rangle$, observe that: 
\begin{itemize}
\item[(1)] it is not the case that an element $x \in \Delta$ is an instance of concept $\Box \neg C$ and another domain element $y \in \Delta$, 
with $y<x$ is an instance of concept $\neg \Box \neg C$;

\item[(2)] given two domain elements $x$ and $y$ such that $x$ and $y$ have different ranks (for instance, $k_\emme(x)=i$, $k_\emme(y)=j$ and $i<j$),  
if they are instances of exactly the same
concepts of the form  $\Box \neg C$ 
(i.e., $x \in (\Box \neg C)^I$ iff $y \in (\Box \neg C)^I$)
for all concepts $C$ occurring in the KB, 
then $y$ can be assigned the same rank as $x$ 
without changing the set of concepts 
of which $y$ is an instance. Note that $\tip$ cannot occur in the scope of a $\Box$ modality.
\end{itemize}
By changing the rank of (possibly infinite many) domain elements according item (2), 
we can transform any $\shiqrt$ model into another $\shiqrt$ model $\emme'=  \langle \Delta, {<'}, I' \rangle$
where each domain element has a finite rank.

For each domain element $x\in \Delta$, let 
$$x^\emme_\Box = \{ \Box \neg C \mid x \in ( \Box \neg C)^I\}$$
We let $I'=I$ and we define  $<'$ by the following ranking function, for any $y\in \Delta$:
\begin{center}
$k_{\emme'}(y) = min \{k_{\emme}(x) \mid x \in \Delta \mbox{ and } x^\emme_\Box = y^\emme_\Box\}$
\end{center}
Observe that $k_{min}(y)$ is well-defined for any element $y\in \Delta$  
(a set of ordinals has always a least element). 
We can show that $\emme' \models$ KB. 
Since $I'$ is the same as $I$ in $\emme$, it follows immediately that 
$\emme'$ satisfies strict concept inclusions, role inclusions and ABox assertions.
Also, for each transitive role $R$, $R^{I'}$ is transitive (as $R^I$ is transitive).

We prove that $\emme'\models K_D$. Let $\tip(C) \sqset E\in F_D$. Suppose, by absurdum, that $\emme' \not\models \tip(C) \sqset E$, this means that
there is a $z\in \Delta$ such that $z \in (T(C))^{I'}$ and $z \not \in E^{I'}$.
We show that in $\emme$,  $z \in (T(C))^{I'}$ and $z \not \in E^{I'}$.
As $I'=I$, from $z \not \in E^{I'}$ it follows that $z \not \in E^{I}$.
Let  $z \in (T(C))^{I'}$. Then, by definition of $T(C)$ as $C \sqcap \Box \neg C$, it must be that  $z \in (C)^{I'}$ and  $z \in (\Box \neg C)^{I'}$.
Observe that, by construction, $z_\Box^{\emme'} = z_\Box^\emme$, since $z$ has been assigned in $\emme'$ the same rank as 
an element $x$ such that $ x^\emme_\Box = z^\emme_\Box$. Therefore,  $z \in (\Box \neg C)^{I}$.
Also, since $I'=I$,  $z \in (C)^{I}$. Hence,  $z \in ( C \sqcap \Box \neg C)^{I}$, and  $z \in (T(C))^{I}$.
We can then conclude that $\emme \not\models \tip(C) \sqset E$,
against the fact that $\emme$ is a model of KB.

Hence, $\emme'$ is a model of KB.
Similarly, it can be easily shown that if $C$ is satisfiable in $\emme$, i.e. there is an $x\in \Delta$ such that $x\in C^I$,
then $x\in C^{I'}$ and therefore, $C$ is satisfiable in $\emme'$.

Observe that, in $\emme'$, any pair of domain element with different ranks cannot be instances of the same concepts $\Box \neg C$
for all the $C$ occurring in the KB  (not containing the $\tip$ operator).
This is true, in particular, for the pairs  $v$ and $w$ of domain elements with adjacent ranks, i.e., 
such that $k_\emme(v)=i+1$ and $k_\emme(w)=i$, for some $i$.
For such a pair, there must be at least a concept $C$ such that $v$ is an instance of $\neg \Box \neg C$ while $w$ is an instance of $\Box \neg C$
(the converse, that  $w$ is an instance of $\neg \Box \neg C$ while $v$ is  an instance of $\Box \neg C$, is not possible by the transitivity of $\Box$, as $w < v$).

As a consequence, for each domain element $w$ with rank $i$, there is at least a concept $C$ occurring in the KB such that: 
all the domain elements with rank $i+1$ are instances of $\neg \Box \neg C$, while $w$ is an instance of $\Box \neg C$.
Informally, the number of $\Box$ formulas of which a domain element is an instance increases, when the rank decreases. 
For a given KB, an upper bound $h_{KB}$ on the rank  of all domain elements can thus be determined as the number 
of (sub)concepts occurring in the KB, which is polynomial in the size of the KB.

$\hfill \Box$

\end{proof}
In the following, we can restrict our consideration to models of the KB with finite ranks whose value is less or equal to $h_{KB}$,
the number of  (sub)concepts occurring in the KB (which is polynomial in the size of the KB).


The following theorem says that
reasoning in $\shiqrt$ has the same complexity as reasoning in $\shiq$, i.e. it is in \textsc{ExpTime}.
Its proof is given by providing an encoding of satisfiability in $\shiqrt$ into satisfiability $\shiq$,
which is known to be an  \textsc{ExpTime}-complete problem.

\vspace{0.3cm}

\noindent {\bf Theorem \ref{encoding_shiqrt_shiq}.}
Satisfiability in  $\shiqrt$ is an  \textsc{ExpTime}-complete problem.

\begin{proof} 
{\em (Sketch)}
The hardness comes from the fact that satisfiability in $\shiq$ is  \textsc{ExpTime}-hard.
We show that satisfiability in  $\shiqrt$ can be solved in  \textsc{ExpTime} 
by defining a polynomial reduction of satisfiability in  $\shiqrt$ to satisfiability in $\shiq$.

Let KB=(TBox,ABox) be a knowledge base, and $C_0$ a concept in $\shiqrt$.
We define an encoding (TBox', ABox') of KB and $C'_0$ of $C_0$ in $\shiq$ as follows.

First, we introduce new atomic concepts $Zero$ and $W$ in the language and a new role $R$,
where $R$ is intended to model the relation $<$ of $\shiqrt$ models. We let TBox' contain the inclusions
\begin{center}
$\top \sqsubseteq \leq 1 R.\top$ \ \ \ \ \ \ \ \ \ $\top \sqsubseteq \leq 1 R^-.\top$
\end{center}
so that $R$ allows to represent linear sequences.
We will consider the linear sequences of elements of the domain reachable trough $R^-$ from the $Zero$ elements,
i.e., those sequences  $w_0, w_1, w_2, \ldots$, with $w_0  \in Zero^I $ and $(w_i, w_{i+1})\in (R^-)^I$.
%
Given Proposition \ref{Prop:catene_corte}, we can restrict our consideration to finite linear sequences
with length less or equal to $h$, the number of sub-concepts of the KB (which is polynomial in the size of KB).
We introduce $h$ new atomic concepts $S_1, \ldots, S_h$ such that the instances of $S_i$ are
the domain elements reachable form a $Zero$ element by a chain of length $i$ of $R^-$-successors.
We introduce in TBox' the following inclusions:
\begin{center}
$Zero \sqsubseteq \forall R^-. S_1$ \ \ \ \ \ \ \
$S_1 \sqsubseteq \exists R. Zero$  \ \ \ \ \ \ \
$S_i \sqsubseteq \forall R^-. S_{i+1}$  \ \ \ \ \ \ \
$S_{i+1}\sqsubseteq \exists R. S_i$
\end{center}
$Zero$-elements have no $R$-successor
and $S_h$-elements have no $R$-predecessors.
\begin{center}
$Zero \sqsubseteq \neg \exists R. \top.$  \ \ \ \ \ \ \
$S_h \sqsubseteq \neg \exists R^-. \top.$
\end{center}

All the elements in a sequences $w_0, w_1, w_2, \ldots$, as introduced above, are instances of concept $W$:
\begin{center}
$W \sqsubseteq Zero \sqcup S_1 \sqcup \ldots \sqcup S_n$  \ \ \ \ \ \ \
$ Zero \sqcup S_1 \sqcup \ldots \sqcup S_n \sqsubseteq  W$
\end{center}

From the sequences $w_0, w_1, w_2, \ldots$ starting from $Zero$ elements, we can encode in $\shiq$ the structure of ranked models of $\shiqrt$, by associating rank $i$ to all the elements $w_i$ in $S_i$.

We have to provide an encoding for the inclusions in TBox.
For each $A \sqsubseteq B \in$ TBox, not containing $\tip$, we introduce  
$A \sqsubseteq B $ in TBox'.

For each $\tip(A)$  occurring in the TBox, we introduce a new atomic concept  $\Box_{\neg A}$
and, for each inclusion $\tip(A) \sqsubseteq B \in$ TBox, we add to TBox'
 the inclusion 
 $$ A \sqcap \Box_{\neg A} \sqsubseteq B$$

%

To capture the properties of the $\Box$ modality, the following equivalences are introduced in TBox':
\begin{center}
$ \Box_{\neg A} \equiv  \forall R. (\neg A \sqcap  \Box_{\neg A})$

$\top  \sqsubseteq \forall U. (\neg S_i \sqcup  \Box_{\neg A}) \sqcup   \forall U. (\neg S_i  \sqcup \neg \Box_{\neg A})$

\end{center}
for all $i=0, \ldots, h$ and for all concept names $A \in {\cal C}$, where $U$ is the universal role  (which can be defined in $\shiq$  \cite{HorrocksIGPL00}).
The first inclusion, says that if a domain element of rank $i$ is an instance of concept $\Box_{\neg A}$,
the elements of rank $i-1$ (in the same sequence) are instances of both the concepts  ${\neg A}$ and $\Box_{\neg A}$.
(this is to account for the transitivity of the $\Box$ modality).
The second inclusion forces the $S_i$-elements
(i.e. all the domain elements with rank $i$)
to be instances of the same boxed concepts $\Box_{\neg A}$, for all $A \in  {\cal C}$.

For each named individual $a \in N_I$, we add to ABox' the assertion $W(a)$,
to guarantee the interpretation of $a$ to be a $W$-element.

For all the assertions $C_R(a)$ in ABox, we add $C_R(a)$ to ABox'.
For all the assertions $\tip(C)(a)$ in ABox, we add $(A \ \sqcap \ \Box_{\neg A})(a)$ to ABox'.
For all the assertions $R(a,b) \in$ABox, we add $R(a,b)$ to ABox'.

Given a $\shiqrt$ concept $C_0$, whose size is assumed to be polynomial in the size of the KB,
we encode by introducing the following $\shiq$ concept $C'_0$
$$\exists U. (W \sqcap [C_0] )$$
where $U$ is the universal role  
and 
$[C_0]$ is obtained from $C_0$ by replacing  each occurrence of $\tip(A)$ in $C_0$
with $A \sqcap \Box_{\neg A}$. $[C_0]$ is a $\shiq$ concept
and we require it to be satisfied in some $W$-element.
We can then prove the following:

\begin{itemize}
\item
The size of KB' and size of $C'_0$ are polynomial in the size of KB.


\item
Concept $C_0$ is satisfiable with respect to KB in $\shiqrt$ if and only if $C'_0 $ is satisfiable  with respect to KB' in $\shiq$.
\end{itemize}
The proof of this result can be done by showing that a $\shiqrt$ model of KB satisfying $C_0$ can be transformed into a 
$\shiqrt$ model of KB' satisfying $C'_0$. And vice-versa.

We can therefore conclude that the satisfiability problem in $\shiqrt$ can be polynomially reduced to the satisfiability problem in $\shiq$,
which is in  \textsc{ExpTime}.
$\hfill \bbox$
\end{proof}

\section{APPENDIX: Well-founded relations}\label{appendicenicola}

A few definitions.

\begin{definition}
Let $S$ be a non-empty set and $<^*$ a transitive, irreflexive relation on $S$ (a strict pre-order). 
Let $U \subseteq S$, with $U\not=\emptyset$, we say that $x\in S$ is a\emph{ minimal element} of $U$ with respect to $<^*$ if it holds:
\begin{quote}
$x\in U$ and $\forall y \in U$ we have $y \not<^* x$.
\end{quote}
Given $U \subseteq S$, we denote by $Min_{<^*}(U)$ the set of minimal elements of $U$ with respect to $<^*$. 
\end{definition}

\begin{definition}
Let $S$ and $<^*$ as in previous definition. We say that $<^*$ \emph{is well-founded } on $S$ if for every non-empty $U \subseteq S$, we have 
$Min_{<^*}(U)\not = \emptyset$.
\end{definition}

\begin{proposition}
Let $S$ and $<^*$ as above. The following are equivalent:
\begin{enumerate}
	\item $<^*$ is well-founded on $S$;
	\item there are no infinite descending chains: $\ldots x_{i+1} <^* x_i <^* \ldots <^* x_0$ of elements of $S$.
\end{enumerate}
\end{proposition}
\begin{proof}
\begin{itemize}
	\item 
$(1) \Rightarrow (2)$. Suppose that $<^*$ is well-founded on $S$ and by absurd that there is an infinite descending chain  $\ldots x_{i+1} <^* x_i <^* \ldots <^* x_0$ of elements of $S$. Let $U$ be the set of elements of such a chain. Clearly for every $x_i\in U$ there is a $x_j\in U$ with $x_j <^* x_i$. But this means that $Min_{<^*}(U)=\emptyset$ against the hypothesis that $<^*$ is well-founded on $S$.

\item
$(2) \Rightarrow (1)$. Suppose by absurd that for a non-empty $U\subseteq S$, we have that $Min_{<^*}(U)=\emptyset$. Thus:
$$\forall x \in U \ \ \exists y \in U \ y <^* x$$
We can \emph{assume} that there is a function $f: U \longrightarrow U$ such that $f(x) <^* x$. 

[If $S$ is enumerable then $f$ can be defined by means of an enumeration of  $S$ (e.g. take the smallest $ y <^* x$ in the enumeration); otherwise and more generally, by using the axiom of choice we can proceed as follows:  
given $x\in U$, let $U_{\downarrow x} = \{y\in U \mid y <^* x\}$, thus for every $x \in U$ the set $U_{\downarrow x}$ is non-empty. 
By \emph{the axiom of choice}, there is a function $g$:
$$g: \{U_{\downarrow x} \mid x\in U\}  \longrightarrow \bigcup_{x\in U} U_{\downarrow x} \ (= U) $$
such that for every $x\in U$, $g(U_{\downarrow x})\in U_{\downarrow x}$. We then define $f(x) = g(U_{\downarrow x})$.] 

Since $f(x) < x$ we also have $f(f(x)) <^* f(x) <^*x$ and so on. Using the notation $f^i(x)$ for the 
$i$ - iteration of $f$, we can immediatly define an infinite descending chain by fixing $x_0\in U$ and by taking $x_i = f^i(x_0)$ for all $i> 0$. 
\end{itemize}
$\hfill \bbox$
\end{proof}

\begin{theorem}
Let $S$ be a non-empty set and $<^*$ a binary relation on $S$. The following are equivalent:
\begin{enumerate}
	\item $<^*$ is (i) irreflexive, (ii) transitive, (iii) modular, (iv) well-founded.
	\item there exists a function $k: S \longrightarrow Ord$ such that $x <^* y$ iff $k(x) < k(y)$ (where Ord is the set of ordinals). 
\end{enumerate}
\end{theorem}
\begin{proof}

\begin{itemize}
	\item $(2) \Rightarrow (1)$. Suppose that there is a function $k: S \longrightarrow Ord$ such that $x <^* y$ iff $k(x) < k('y)$. We can easily check that properties (i)--(iv) holds: irreflexivity and transitivity are immediate. For (iii) modularity: let $x <^* y$ and $z$ be any element in $S$. 
Suppose that $ x \not <^* z$, thus $k(x) \not < k(z)$; then it must be either $k(x) = k(z)$ or $k(z) < k(x)$, whence $k(z) < k(y)$ in both cases, thus $z <^* y$. 

For (iv) well-foundedness, suppose by absurd that there is a non-empty $U \subseteq S$ such that $Min_{<^*}(U)=\emptyset$, then for every $x \in U$ there is $y\in U$ such that $y <^* x$. Let us consider the image of $U$ under $k$: $U_k = \{k(x) \mid x\in U\}$. The set of ordinals $U_k$ has a \emph{least} element, say $\beta\in Ord$ (this by  property of ordinals: \emph{every non-empty set of ordinals has a least element}).  
Let $z\in U$ such that $k(z) = \beta$. By hypothesis, there is $y\in U$ such that $y <^* u$, but then $k(y) \in U_k$ and $k(y) < \beta$, against the fact that $\beta$ is the \emph{least} ordinal in $U_k$.

\item 
$(1) \Rightarrow (2)$ (Sketch). Suppose that $<^*$ satisfies properties (i)--(iv). Let us consider the following sequence of sets indexed on Ordinals:
\begin{quote}
$S_{\alpha} = S - \bigcup_{\beta < \alpha} A_\beta$\\
$A_{\alpha} = Min_{<^*}(S_\alpha)$
\end{quote}
Thus $S_0 = S$ and $A_0 = Min_{<^*}(S)$. 
Observe that if $S_{\alpha}\not=\emptyset$ then also $A_{\alpha}\not=\emptyset$ (by well-foundness); moreover the sequence is decreasing: $S_\alpha \subset S_\beta$ for $\beta < \alpha$. But for cardinality reasons there must be a least ordinal $\lambda$ such that $S_{\lambda} = \emptyset$, this means  that $S_{\lambda} = S - \bigcup_{\beta < \lambda} A_\beta = \emptyset$,  so that we get 
$$S = \bigcup_{\beta < \lambda} A_\beta$$
It can be easily shown that: 
\begin{itemize}
	\item for $\alpha<\beta<\lambda $, $\forall x\in A_{\alpha}, \forall y\in A_{\beta} \ x <^* y$, and also $A_{\alpha}\cap A_{\beta} = \emptyset$
	\item for each $x\in S$, there exists a \emph{unique} $A_{\alpha}$ with $\alpha < \lambda$ such that $x\in A_{\alpha}$
	\item  $x <^* y$ iff for some $\alpha, \beta < \lambda$  $x \in A_{\alpha}$ and $y\in A_{\beta}$ and $\alpha <\beta$.
\end{itemize}
We can then define $k(x) =$ the unique $\alpha$ such that $x\in A_{\alpha}$ and the result follows.
\end{itemize}
$\hfill \bbox$
\end{proof}

\end{appendix}

\end{document}